\pdfoutput=1

\documentclass{article}
\usepackage{ijcai21}

\usepackage{times}
\usepackage{soul}
\usepackage{url}
\usepackage[hidelinks]{hyperref}
\usepackage[utf8]{inputenc}
\usepackage[small]{caption}
\usepackage{graphicx}
\usepackage{amsmath}
\usepackage{amsthm}
\usepackage{booktabs}
\usepackage{algorithm}
\usepackage{algorithmic}
\newtheorem{theorem}{Theorem}

\usepackage{macros} %

\title{The Surprising Power of Graph Neural Networks with\\ Random Node Initialization}

\author{
Ralph Abboud$^1$
\and
{\.I}smail {\.I}lkan Ceylan$^1$\and
Martin Grohe$^{2}$\And
Thomas Lukasiewicz$^1$
\affiliations
$^1$University of Oxford\\
$^2$RWTH Aachen University\\
}

\begin{document}
\maketitle

\begin{abstract}
Graph neural networks (GNNs) are effective models for representation learning on relational data. 
However, standard GNNs are limited in their expressive power, as they cannot distinguish graphs beyond the capability of the Weisfeiler-Leman graph isomorphism heuristic. 
In order to break this expressiveness barrier, GNNs have been enhanced with \emph{random node initialization (RNI)}, where the idea is to train and run the models with randomized initial node features. 
In this work, we analyze the expressive power of GNNs with RNI, and prove that these models are \emph{universal}, a first such result for GNNs not relying on computationally demanding higher-order properties.
This universality result holds even with partially randomized initial node features, and preserves the \emph{invariance} properties of GNNs in expectation.
We then empirically analyze the effect of RNI on GNNs, based on carefully constructed datasets. 
Our empirical findings support the superior performance of GNNs with RNI over standard GNNs.
\end{abstract}

\section{Introduction}
Graph neural networks (GNNs) \cite{Scarselli09,Gori2005} are neural architectures designed for learning functions over graph domains, and naturally encode desirable properties such as permutation invariance (resp., equivariance) relative to graph nodes, and node-level computation based on message passing. These properties provide GNNs with a strong inductive bias, enabling them to effectively learn and combine both local and global graph features \cite{BattagliaGraphNetworks}. GNNs have been applied to a multitude of tasks, ranging from protein classification \cite{GilmerSRVD17} and synthesis \cite{YouLYPL18}, protein-protein interaction \cite{FoutBSB17}, and social network analysis \cite{HamiltonYL17}, to recommender systems \cite{YingHCEHL18} and combinatorial optimization \cite{BengioTourDHorizon}. %

While being widely applied, popular GNN architectures, such as message passing neural networks (MPNNs), are limited in their expressive power. Specifically, MPNNs are at most as powerful as the Weisfeiler-Leman (1-WL) graph isomorphism heuristic \cite{MorrisAAAI19,Keyulu18}, and thus cannot discern between several families of non-isomorphic graphs, e.g., sets of regular graphs \cite{CaiFI92}. 
To address this limitation, alternative GNN architectures with provably higher expressive power, such as $k$-GNNs~\cite{MorrisAAAI19} and invariant (resp., equivariant) graph networks~\cite{MaronInvEqui19}, have been proposed. These models, which we refer to as \emph{higher-order GNNs}, are inspired by the generalization of 1-WL to $k-$tuples of nodes, known as $k$-WL  \cite{CaiFI92}. 
While these models are very expressive, they are computationally very demanding. As a result, MPNNs, despite their limited expressiveness, remain the standard for graph representation learning. 

In a rather recent development, MPNNs have achieved empirical improvements using \emph{random node initialization} (RNI), in which initial node embeddings are randomly set. Indeed, RNI enables MPNNs to detect \emph{fixed} substructures,  so extends their power beyond 1-WL, and also allows for a better approximation of a class of combinatorial problems \cite{SatoRandom2020}. While very important, these findings do not explain the overall theoretical impact of RNI on GNN learning and generalization for \emph{arbitrary} functions. 

In this paper, we thoroughly study the impact of RNI on MPNNs. Our main result states that MPNNs enhanced with RNI are \emph{universal}, and thus can approximate every function defined on graphs of any fixed order. This follows from a logical characterization of the expressiveness of MPNNs \cite{BarceloKM0RS20} combined with an argument on order-invariant definability. 
Importantly, MPNNs enhanced with RNI preserve the \emph{permutation-invariance} of MPNNs in expectation, and possess a strong inductive bias. 
Our result strongly contrasts with 1-WL limitations of deterministic MPNNs, and provides a foundation for developing expressive and memory-efficient MPNNs with strong inductive bias.

To verify our theoretical findings, we carry out a careful empirical study. We design \EXP, a synthetic dataset requiring 2-WL expressive power for models to achieve above-random performance, and run MPNNs with RNI on it, to observe \emph{how well} and \emph{how easily} this model can learn and generalize. Then, we propose \CEXP, a modification of \EXP with partially 1-WL distinguishable data, and evaluate the same questions in this more variable setting. Overall, the contributions of this paper are as follows:
\begin{enumerate}[-,leftmargin=*]

	\item We prove that MPNNs with RNI are universal, while being permutation-invariant in expectation. This is a significant improvement over the 1-WL limit of standard MPNNs and, to our knowledge, a first universality result for  memory-efficient GNNs.
	
	\item We introduce two carefully designed datasets, \EXP and \CEXP, based on graph pairs only distinguishable by 2-WL or higher, to rigorously evaluate the impact of RNI.
	
	\item  We analyze the effects of RNI on MPNNs on these datasets, and observe that (i)~MPNNs with RNI closely match the performance of higher-order GNNs, (ii)~the improved performance of MPNNs with RNI comes at the cost of slower convergence, and (iii)~partially randomizing initial node features improves model convergence and accuracy.
	\item We additionally perform the same experiments with analog sparser datasets, with longer training, and observe a similar behavior, but more volatility.
\end{enumerate}
The proof of the main theorem, as well as further details on datasets and experiments, can be found in the appendix of this paper.

\section{Graph Neural Networks}

Graph neural networks (GNNs) \cite{Gori2005,Scarselli09} are neural models for learning functions over graph-structured data.
In a GNN, graph nodes are assigned vector representations, which are updated iteratively through series of \emph{invariant} or \emph{equivariant} computational layers. 
Formally, a function $f$ is \emph{invariant} over graphs if, for isomorphic graphs $G,H\,{\in}\,\CG$ it holds that $f(G)\,{=}\,f(H)$. Furthermore, a function $f$ mapping a graph $G$ with vertices $V(G)$ to vectors ${\boldsymbol x\in\Rbb^{|V(G)|}}$ is \emph{equivariant} if,  for every permutation $\pi$ of $V(G)$, it holds that ${f(G^\pi)=f(G)^\pi}$.

\subsection{Message Passing Neural Networks} 
In MPNNs \cite{GilmerSRVD17}, node representations aggregate \emph{messages} from their neighboring nodes, and use this information to iteratively update their representations. Formally, given a node $x$, its vector representation $v_{x,t}$ at time $t$, and its
neighborhood $N(x)$, an update can be written as:
\[
v_{x,t+1} = combine \Big(v_{x,t},aggregate\big(\{v_{y,t}|~y \in N(x)\}\big)\Big),
\]
where \emph{combine} and \emph{aggregate} are functions, and \emph{aggregate} is typically permutation-invariant. 
Once message passing is complete, the final node representations are then used to compute target outputs. Prominent MPNNs include graph convolutional networks~(GCNs)~\cite{Kipf16} and graph attention networks~(GATs)~\cite{VelickovicCCRLB18}.

It is known that standard MPNNs have the same power as the 1-dimensional Weisfeiler-Leman algorithm (1-WL) \cite{Keyulu18,MorrisAAAI19}. This entails that graphs (or nodes) cannot be distinguished by MPNNs if $1$-WL does not distinguish them. For instance, 1-WL cannot distinguish between the graphs $G$ and $H$, shown in Figure \ref{fig:indistinguishable}, despite them being clearly non-isomorphic. Therefore, MPNNs cannot learn functions with different outputs for $G$ and $H$. 

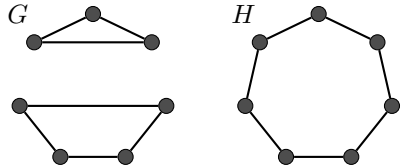
\begin{figure}[t!]
	\centering
	\begin{tikzpicture}[
	vertex/.style = {draw,fill=black!70,circle,inner sep=0pt, minimum height= 2mm},
	]
	\begin{scope}
	\node[vertex] (v0) at (90.00:1cm) {};
	\node[vertex] (v1) at (141.43:1cm) {};
	\node[vertex] (v2) at (192.86:1cm) {};
	\node[vertex] (v3) at (244.29:1cm) {};
	\node[vertex] (v4) at (295.71:1cm) {};
	\node[vertex] (v5) at (347.14:1cm) {};
	\node[vertex] (v6) at (398.57:1cm) {};
	\draw[thick] (v0) edge (v1) edge (v6) (v1) edge (v6) (v2) edge
	(v3) edge (v5) (v4) edge (v5) edge (v3);
	\node at (-1,1) {$G$};
	\end{scope}
	
	\begin{scope}[xshift=3cm]
	\node[vertex] (v0) at (90.00:1cm) {};
	\node[vertex] (v1) at (141.43:1cm) {};
	\node[vertex] (v2) at (192.86:1cm) {};
	\node[vertex] (v3) at (244.29:1cm) {};
	\node[vertex] (v4) at (295.71:1cm) {};
	\node[vertex] (v5) at (347.14:1cm) {};
	\node[vertex] (v6) at (398.57:1cm) {};
	\draw[thick] (v0) edge (v1) edge (v6) (v2) edge (v1) edge (v3)
	(v4) edge (v3) edge (v5) (v5) edge (v6);
	\node at (-1,1) {$H$};
	\end{scope}
	\end{tikzpicture}
	\caption{$G$ and $H$ are indistinguishable by $1$-WL} %
	\label{fig:indistinguishable}
\end{figure}
Another somewhat trivial limitation in the expressiveness of MPNNs is that information is only propagated along edges, and hence can never be shared between distinct connected components of a graph \cite{BarceloKM0RS20,Keyulu18}. 
An easy way to overcome this limitation is by adding \emph{global readouts}, that is, permutation-invariant functions that aggregate the current states of all nodes. Throughout the paper, we therefore focus on MPNNs with global readouts, referred to as \emph{ACR-GNNs} \cite{BarceloKM0RS20}. 

\subsection{Higher-order Graph Neural Networks}
We now present the main classes of higher-order GNNs.

\paragraph{Higher-order MPNNs.} The $k-$WL hierarchy has been directly emulated in GNNs, such that these models learn embeddings for \emph{tuples} of nodes, and perform message passing between them, as opposed to individual nodes. This higher-order message passing approach resulted in models such as $k$-GNNs \cite{MorrisAAAI19}, which have $(k-1)$-WL expressive power.\footnote{In the literature, different versions of the Weisfeiler-Leman algorithm have inconsistent dimension counts, but are equally expressive. For example, $(k+1)$-WL and $(k+1)$-GNNs in \cite{MorrisAAAI19} are equivalent to $k$-WL of \cite{CaiFI92,GroheWL}. We follow the latter, as it is the standard in the  literature on graph isomorphism testing.} These models need $O(|V|^k)$ memory to run, leading to \emph{excessive memory requirements}. 

\paragraph{Invariant (resp., equivariant) graph networks.} Another class of higher-order GNNs is invariant (resp., equivariant) graph networks \cite{MaronInvEqui19}, which represent graphs as a tensor, and implicitly pass information between nodes through invariant (resp., equivariant) computational blocks.  %
Following intermediate blocks, \emph{higher-order} tensors are typically returned, and the order of these tensors correlates directly with the expressive power of the overall model. Indeed, invariant networks \cite{MaronFSL19}, and later equivariant networks \cite{KerivenP19}, are shown to be universal, but with tensor orders of $O(|V|^2)$, where $|V|$ denotes the number of graph nodes. Furthermore, invariant (resp., equivariant) networks with intermediate tensor order $k$ are shown to be equivalent in power to $(k-1)$-WL \cite{MaronBSL19}, which is strictly more expressive as $k$ increases \cite{CaiFI92}. Therefore, universal higher-order models require \emph{intractably-sized intermediate tensors} in practice. 

\paragraph{Provably powerful graph networks.} A special class of invariant GNNs is provably powerful graph networks (PPGNs)\cite{MaronBSL19}. PPGNs are based  on ``blocks'' of multilayer perceptrons (MLPs) and matrix multiplication, which theoretically have 2-WL expressive power, and only require memory $O(|V|^2)$ (compared to $O(|V|^3)$ for 3-GNNs). However, PPGNs theoretically require \emph{exponentially many samples} in the number of graph nodes to learn necessary functions for 2-WL expressiveness \cite{PunyLRGA}.

\section{MPNNs with Random Node Initialization}
\label{sec:RNIUniversal}
We present the main result of the paper, showing that RNI makes MPNNs universal, in a natural sense. 
Our work is a first positive result for the universality of MPNNs. This result is not based on a new model, but rather on random initialization of node features, which is widely used in practice, and in this respect, it also serves as a theoretical justification for models that are empirically successful.

\subsection{Universality and Invariance}
It may appear somewhat surprising, and even counter-intuitive, that randomly initializing node features on its own would deliver such a gain in expressiveness.
In fact, on the surface, random initialization no longer preserves the invariance of MPNNs, since the result of the computation of an MPNN with RNI not only depends on the structure (i.e., the isomorphism type) of the input graph, but also on the random initialization.
The broader picture is, however, rather subtle, as we can view such a model as computing a random variable (or as generating an output distribution), and this random variable would still be invariant. 
This means that the outcome of the computation of an MPNN with RNI does still \emph{not} depend on the specific representation of the input graph, which fundamentally maintains invariance. Indeed, the mean of random features, in expectation, will inform GNN predictions, and is identical across all nodes, as randomization is i.i.d. However, the variability between different samples and the variability of a random sample enable graph discrimination and improve expressiveness. Hence, in expectation, all samples fluctuate around a unique value, preserving invariance, whereas sample variance improves expressiveness.

Formally, let $\mathcal G_n$ be the class of all $n$-vertex graphs, i.e., graphs that consist of at most $n$ vertices, and let ${f:\Gmc_n\to\Rbb}$. 
We say that a randomized function $\Xmc$ that associates with every graph $G\in \Gmc_n$ a random variable $\Xmc (G)$ is an \emph{$(\epsilon,\delta)$-approximation} of $f$ if for all $G\in\Gmc_n$ it holds that ${\Pr\big(|f(G)-\Xmc(G)|\le\epsilon\big)\ge 1-\delta}$. 
Note that an MPNN $\Nmc$ with RNI computes such functions $\Xmc$. If $\Xmc$ is computed by $\Nmc$, we say that \emph{$\Nmc$ $(\epsilon,\delta)$-approximates $f$}.
\begin{theorem}[Universal approximation]
\label{theo:uni}
	Let $n\ge 1$, and let $f:\mathcal G_n\to\mathbb R$ be invariant. Then, for all
	$\epsilon,\delta>0$, there is an MPNN with RNI that $(\epsilon,\delta)$-approximates $f$.
\end{theorem}
For ease of presentation, we state the theorem only for real-valued functions, but note that it can be extended to equivariant functions. The result can also be extended to weighted graphs, but then the function $f$ needs to be continuous.

\subsection{Result Overview}
To prove \Cref{theo:uni}, we first show that MPNNs with RNI can capture arbitrary Boolean functions, by building on the result of \cite{BarceloKM0RS20}, which states that any logical sentence in $\LC^2$ can be captured by an MPNN (or, by an ACR-GNN in their terminology). The logic $\LC$ is the extension of first-order predicate logic using counting quantifiers of the form $\exists^{\ge k}x$ for $k\ge 0$, where $\exists^{\ge k}x\phi(x)$ means that there are at least $k$ elements $x$ satisfying $\phi$, and $\LC^2$  is the two-variable fragment of $\LC$.

We establish that any graph with identifying node features, which we call \emph{individualized graphs}, can be represented by a sentence in $\LC^2$. Then, we extend this result to sets of individualized graphs, and thus to Boolean functions mapping these sets to True, by showing that these functions are represented by a $\LC^2$ sentence, namely, the disjunction of all constituent graph sentences. 
Following this, we provide a construction with node embeddings based on RNI, and show that RNI individualizes input graphs w.h.p. 
Thus, RNI makes that MPNNs learn a Boolean function over individualized graphs w.h.p. Since all such functions can be captured by a sentence in $\LC^2$, and an MPNN can capture any Boolean function, MPNNs with RNI can capture arbitrary Boolean functions. 
Finally, the result is extended to real-valued functions via a natural mapping, yielding universality.

The concrete implications of \Cref{theo:uni} can be summarized as follows. First, MPNNs with RNI can distinguish individual graphs with an embedding dimensionality polynomial in the inverse of desired confidence $\delta$ (namely, $O(n^2 \delta^{-1})$, where $n$ is the number of graph nodes). 
Second, universality also holds with partial RNI, and even with only one randomized dimension. 
Third, the theorem is adaptive and tightly linked to the descriptive complexity of the approximated function. That is, for a more restricted class of functions, there may be more efficient constructions than the disjunction of individualized graph sentences, and our proof does not rely on a particular construction. 
Finally, our construction provides a \emph{logical characterization}for MPNNs with RNI, and substantiates how randomization improves expressiveness. This construction therefore also enables a more logically grounded theoretical study of randomized MPNN models, based on particular architectural or parametric choices.

Similarly to other universality results, Theorem~\ref{theo:uni} can potentially result in very large constructions. This is a simple consequence of the generality of such results: \Cref{theo:uni} applies to families of functions, describing problems of \emph{arbitrary} computational complexity, including problems that are computationally hard, even to approximate.
Thus, it is more relevant to empirically verify the formal statement, and test the capacity of MPNNs with RNI relative to higher-order GNNs. 
Higher-order GNNs typically suffer from prohibitive space requirements, but this not the case for MPNNs with RNI, and this already makes them more practically viable. 
In fact, our experiments demonstrate that MPNNs with RNI indeed combine expressiveness with efficiency in practice.

\section{Datasets for Expressiveness Evaluation} 
\label{sec:dataset} 

GNNs are typically evaluated on real-world datasets \cite{KKMMN2016}, which are not tailored for evaluating expressive power, as they do not contain instances indistinguishable by $1$-WL. In fact, higher-order models only marginally outperform MPNNs on these datasets \cite{BenchmarkingGNNs}, which further highlights their unsuitability.
Thus, we developed the synthetic datasets \EXP and \CEXP. \EXP explicitly evaluates GNN expressiveness, and consists of graph instances $\{G_1, \ldots ,G_n$, $H_1, \ldots ,H_n\}$, where each instance encodes  a propositional formula. The classification task is to determine whether the formula is satisfiable (\SAT). %
Each pair $(G_i$, $H_i)$ respects the following properties: 
(i)~$G_i$ and $H_i$ are non-isomorphic, 
(ii)~$G_i$ and $H_i$ have different \SAT outcomes, that is, $G_i$ encodes a satisfiable formula, while $H_i$ encodes an unsatisfiable formula, 
(iii)~$G_i$ and $H_i$ are 1-WL indistinguishable, so are \emph{guaranteed} to be classified in the same way by standard MPNNs, and 
(iv)~$G_i$ and $H_i$ are 2-WL distinguishable, so \emph{can} be classified differently by higher-order GNNs. 

Fundamentally, every $(G_i, H_i)$ is carefully constructed on top of a basic building block, the \emph{core pair}. 
In this pair, both cores are based on propositional clauses, such that one core is satisfiable and the other is not, both \emph{exclusively} determine the satisfiability of $G_i$ (resp., $H_i$), and have graph encodings enabling all aforementioned properties. 
Core pairs and their resulting graph instances in \EXP are \emph{planar} and are also carefully constrained to ensure that they are 2-WL distinguishable. Thus, core pairs are key substructures within \EXP, and distinguishing these cores is essential for a good performance.

Building on \EXP, \CEXP includes instances with varying expressiveness requirements.  Specifically, \CEXP is a standard \EXP dataset where 50\% of all satisfiable graph pairs are made 1-WL distinguishable from their unsatisfiable counterparts, only differing from these by a small number of added edges.
Hence, \CEXP consists of 50\% ``corrupted'' data, distinguishable by MPNNs and labelled \Corrupt,  and 50\% unmodified data, generated analogously to \EXP, and requiring expressive power beyond 1-WL, referred to as \EXPTwo. 
Thus, \CEXP contains the same core structures as \EXP, but these lead to different \SAT values in \EXPTwo and \Corrupt, which makes the learning task more challenging than learning \EXPTwo or \Corrupt in isolation.

\section{Experimental Evaluation} 

In this section, we first evaluate the effect of RNI on MPNN expressiveness based on \EXP, and compare against established higher-order GNNs. 
We then extend our analysis to \CEXP. 
Our experiments use the following models: 
\paragraph{1-WL GCN (\OneGCN).} A GCN with 8 distinct message passing iterations, ELU non-linearities \cite{ClevertUH15}, 64-dimensional embeddings, and deterministic learnable initial node embeddings indicating node type. 
	This model is guaranteed to achieve 50\% accuracy on \EXP. 
	
\paragraph{GCN - Random node initialization (\GCNRNI).} A 1-GCN enhanced with RNI. We evaluate this
	model with four initialization distributions, namely, 
	the standard normal distribution $\mathcal N(0,1)$ (N),
	the uniform distribution over $[-1,1]$ (U), Xavier normal (XN), and the Xavier uniform distribution (XU)   \cite{GlorotB10}. %
	We denote the respective models \GCNRNI($D$), where $D \in \{\text{N,\,U,\,XN,\,XU}\}$.  
	
\paragraph{GCN - Partial RNI (GCN-$x$\%RNI).}
	A \GCNRNI~model, where $\floor{\frac{64x}{100}}$ dimensions are initially randomized, and all remaining dimensions are set deterministically from one-hot representation of the two input node types (literal and disjunction). We set $x$ to the extreme values 0 and 100\%, 50\%, as well as  near-edge cases of 87.5\% and 12.5\%, respectively.
	
\paragraph{PPGN.} A higher-order GNN with 2-WL expressive power \cite{MaronBSL19}. We set up PPGN using its original implementation, and use its default configuration of eight 400-dimensional computational blocks.
	
\paragraph{\OneTwoThreeLocal.} A higher-order GNN~\cite{MorrisAAAI19} emulating $2$-WL on 3-node tuples. \OneTwoThreeLocal operates at increasingly coarse granularity, starting with single nodes and rising to 3-tuples. This model uses a \emph{connected} relaxation of 2-WL, which slightly reduces space requirements, but comes at the cost of some theoretical guarantees.  We set up \OneTwoThreeLocal with 64-dimensional embeddings, 3 message passing iterations at level 1, 2 at level 2, and 8 at level 3.
	
\paragraph{\ThreeGNNFull.} A GCN analog of the 
	\emph{full} 2-WL procedure over 3-node tuples, thus preserving all theoretical guarantees.

\subsection{How Does RNI Improve Expressiveness?}

In this experiment, we evaluate GCNs using different RNI settings on \EXP, and compare with standard GNNs and higher-order models.
Specifically, we generate an \EXP dataset consisting of 600 graph pairs. Then, we evaluate all models on \EXP using 10-fold cross-validation. We train \ThreeGNNFull for 100 epochs per fold, and all other systems for 500 epochs, and report \emph{mean test accuracy} across all folds.

\begin{table}[t!]
	\centering
	\begin{tabular}{HlHc} 
		\toprule
		\multicolumn{2}{c}{Model} & Training Accuracy (\%) & Test Accuracy (\%)\\
		\midrule
		\multirow{4}{*}{\textit{\GCNRNI}} & \GCNRNI(U) & 97.9 $\pm$ 0.64 & 97.3 $\pm$ 2.55 \\ 
		& \textbf{\GCNRNI(N)} & \textbf{98.0 $\pm$ 0.53} & \textbf{98.0 $\pm$ 1.85} \\ 
		& \GCNRNI(XU) & 97.0 $\pm$ 0.83 & 97.0 $\pm$ 1.43 \\
		& \GCNRNI(XN) & 97.0 $\pm$ 1.20 & 96.6 $\pm$ 2.20 \\ 
		\cmidrule{1-4}
		\multirow{3}{*}{\textit{HO-GNNs}} & PPGN & 50.0 & 50.0\\ 
		& \OneTwoThreeLocal & 50.0 & 50.0\\ 
		& \textbf{\ThreeGNNFull} & \textbf{99.9 $\pm$ 0.002} & \textbf{99.7 $\pm$ 0.004} \\ 
		\bottomrule
	\end{tabular}
		\caption{Accuracy results on \EXP.}
	\label{tab:exp1Results}
\end{table}%

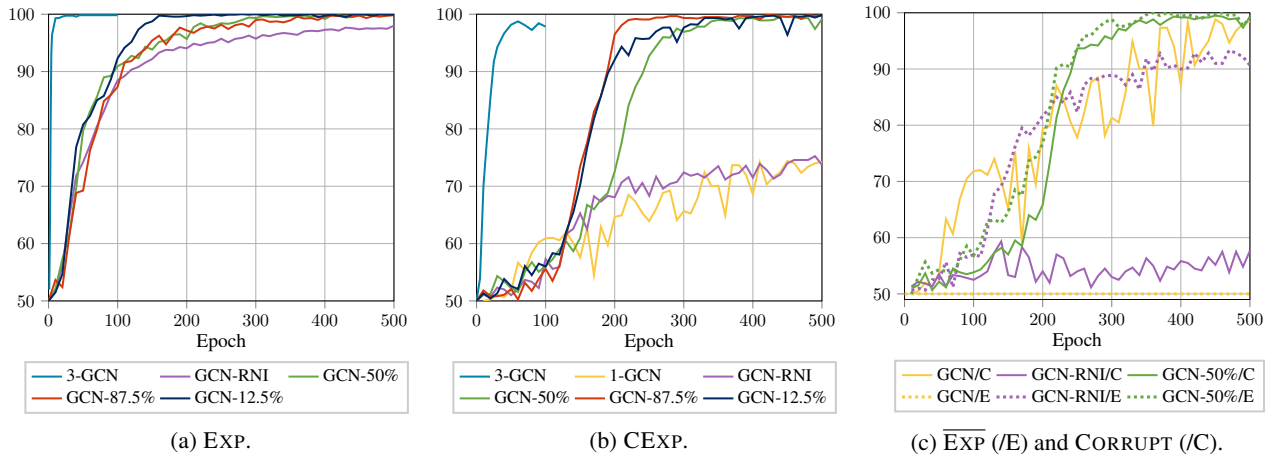
\begin{figure*}[t!]
	\centering
	\begin{subfigure}{.32\textwidth}
		\centering
		\begin{tikzpicture}[scale=0.67]

\definecolor{ThreeGCN}{rgb}{0,0.516,0.66}
\definecolor{GCNStandard}{rgb}{0.98,0.79,0.26}
\definecolor{GCNRNI12}{rgb}{0,0.16,0.39}
\definecolor{GCNRNI50}{rgb}{0.40,0.66,0.26}
\definecolor{GCNRNI87}{rgb}{0.84,0.22,0.05}
\definecolor{GCNRNIFull}{rgb}{0.63,0.39,0.7}

\begin{axis}[
legend cell align={left},
legend columns=3,
legend entries={{\ThreeGNNFull},{\GCNRNI}, {GCN-50\%}, {GCN-87.5\%}, {GCN-12.5\%}},
legend style={at={(-0.07, -0.21)}, font=\normalsize, line width=0.4mm, anchor=north west, draw=white!80.0!black},
label style={font=\large},
tick label style={font=\normalsize},
tick align=outside,
tick pos=left,
x grid style={white!69.01960784313725!black},
xlabel={Epoch},
xmajorgrids,
xmin=0, xmax=500, 
y grid style={white!69.01960784313725!black},
ymajorgrids,
ymin=50, ymax=100
]

\addplot [line width=0.4mm, ThreeGCN]
table [row sep=\\]{%
1 52\\
2 62.41\\
3 81.1\\
4 93.1\\
5 96.4\\
10 99.3\\
15 99.33\\
20 99.58\\
25 99.75\\
30 99.75\\
35 99.75\\
40 99.58\\
45 99.83\\
90 99.83\\
95 99.83\\
100 99.83\\
};

\addplot [line width=0.4mm, GCNRNIFull]
table [row sep=\\]{%
0 50\\
10 51.7\\
20 56.2\\
30 64.6\\
40 71.85\\
50 74.21\\
60 77.18\\
70 80.45\\
80 83.06\\
90 85.98\\ 
100 88.48\\
110 89.21\\
120 90.33\\
130 90.8\\
140 91.56\\
150 92.19\\
160 93.31\\
170 93.79\\
180 93.71\\
190 94.27\\
200 94.05\\
210 94.86\\
220 94.59\\
230 95\\
240 95.12\\
250 95.72\\
260 95.10\\
270 95.70\\
280 95.87\\
290 96.27\\
300 95.73\\
310 96.36\\
320 96.17\\
330 96.53\\
340 96.73\\
350 96.59\\
360 96.40\\
370 97.01\\
380 97.11\\
390 97.04\\
400 97.28\\
410 97.36\\
420 97.13\\
430 97.66\\
440 97.56\\
450 97.37\\
460 97.54\\
470 97.48\\
480 97.55\\
490 97.46\\
500 98.01\\
};

\addplot [line width=0.4mm, GCNRNI50]
table [row sep=\\]{%
0 50\\
10 52.17\\
20 57.33\\
30 61.58\\
40 69.75\\
50 79.67\\
60 83.08\\
70 85.58\\
80 89.00\\
90 89.25\\ 
100 90.92\\
110 91.58\\
120 92.75\\
130 92.33\\
140 94.33\\
150 93.83\\
160 95.17\\
170 95.42\\
180 96.42\\
190 96.50\\
200 95.67\\
210 97.75\\
220 98.42\\
230 97.83\\
240 98.00\\
250 98.42\\
260 98.33\\
270 98.58\\
280 99.00\\
290 99.42\\
300 99.33\\
310 99.67\\
320 99.58\\
330 99.50\\
340 99.58\\
350 99.67\\
360 99.58\\
370 99.75\\
380 99.58\\
390 99.67\\
400 99.17\\
410 99.83\\
420 99.83\\
430 99.42\\
440 99.50\\
450 99.83\\
460 99.92\\
470 99.75\\
480 99.83\\
490 99.75\\
500 99.92\\
};

\addplot [line width=0.4mm, GCNRNI87] %
table [row sep=\\]{%
0 50\\
10 53.67\\
20 52.41\\
30 61.33\\
40 68.83\\
50 69.25\\
60 76.25\\
70 79.83\\
80 84.83\\
90 85.92\\ 
100 87.33\\
110 91.58\\
120 91.83\\
130 93.00\\
140 93.92\\
150 95.42\\
160 96.25\\
170 94.75\\
180 96.33\\
190 97.67\\
200 97.17\\
210 96.75\\
220 97.50\\
230 97.83\\
240 97.50\\
250 98.17\\
260 97.42\\
270 98.08\\
280 98.33\\
290 97.83\\
300 99.00\\
310 99.08\\
320 98.67\\
330 98.75\\
340 98.58\\
350 98.92\\
360 99.58\\
370 99.25 \\
380 99.33 \\
390 99.00 \\
400 99.58 \\
410 99.75 \\
420 99.67 \\
430 99.67 \\
440 99.58 \\
450 99.75 \\
460 99.33 \\
470 99.67 \\
480 99.58 \\
490 99.75 \\
500 99.83 \\
};

\addplot [line width=0.4mm, GCNRNI12] %
table [row sep=\\]{%
0 50\\
10 51.50\\
20 54.67\\
30 65.83\\
40 76.83\\
50 80.75\\
60 82.25\\
70 85.00\\
80 85.75\\
90 88.58\\ 
100 92.33\\
110 94.17\\
120 95.17\\
130 97.33\\
140 98.33\\
150 98.75\\
160 99.75\\
170 99.67\\
180 99.58\\
190 99.58\\
200 99.67\\
210 99.83\\
220 99.75\\
230 99.83\\
240 99.92\\
250 99.75\\
260 99.92\\
270 99.92\\
280 100.0\\
290 99.92\\
300 100.0\\
310 99.83\\
320 99.92\\
330 100.0\\
340 99.75\\
350 99.83\\
360 99.83\\
370 99.92\\
380 100.0 \\
390 99.50 \\
400 99.75 \\
410 99.92 \\
420 99.75 \\
430 99.50 \\
440 99.92 \\
450 100.0 \\
460 99.93 \\
470 99.75 \\
480 100.0 \\
490 99.92 \\
500 99.92 \\
};
\end{axis}
\end{tikzpicture}
		\caption{\EXP.}
		\label{fig:modelConvergence}
	\end{subfigure}%
	\begin{subfigure}{.32\textwidth}
		\centering
		\begin{tikzpicture}[scale=0.67]

\definecolor{ThreeGCN}{rgb}{0,0.516,0.66}
\definecolor{GCNStandard}{rgb}{0.98,0.79,0.26}
\definecolor{GCNRNI12}{rgb}{0,0.16,0.39}
\definecolor{GCNRNI50}{rgb}{0.40,0.66,0.26}
\definecolor{GCNRNI87}{rgb}{0.84,0.22,0.05}
\definecolor{GCNRNIFull}{rgb}{0.63,0.39,0.7}

\begin{axis}[
legend cell align={left},
legend columns=3,
legend entries={{\ThreeGNNFull},{\OneGCN},{\GCNRNI}, {GCN-50\%}, {GCN-87.5\%}, {GCN-12.5\%}},
legend style={at={(-0.07, -0.21)}, font=\normalsize, line width=0.4mm, anchor=north west, draw=white!80.0!black},
label style={font=\large},
tick label style={font=\normalsize},
tick align=outside,
tick pos=left,
x grid style={white!69.01960784313725!black},
xlabel={Epoch},
xmajorgrids,
xmin=0, xmax=500, 
y grid style={white!69.01960784313725!black},
ymajorgrids,
ymin=50, ymax=100
]

\addplot [line width=0.4mm, ThreeGCN]
table [row sep=\\]{%
0 50\\
5 53.91 \\
10 70.08\\
15 78.00 \\
20 85.33 \\
25 91.74 \\
30 94.25 \\
40 96.67 \\
50 98.17 \\
60 98.75 \\
70 98.08 \\
80 97.25 \\
90 98.42 \\ 
100 97.83\\
};

\addplot [line width=0.4mm,GCNStandard]
table [row sep=\\]{
10 50.0\\
20 50.25000000000001\\
30 51.0\\
40 50.66666666666666\\
50 51.83333333333333\\
60 56.66666666666668\\
70 55.333333333333336\\
80 58.41666666666667\\
90 60.25\\
100 60.916666666666664\\
110 61.0\\
120 60.583333333333336\\
130 61.999999999999986\\
140 60.083333333333336\\
150 57.58333333333333\\
160 62.66666666666667\\
170 54.49999999999999\\
180 62.916666666666664\\
190 59.75\\
200 64.66666666666667\\
210 64.91666666666667\\
220 68.5\\
230 67.33333333333333\\
240 65.33333333333333\\
250 63.91666666666668\\
260 66.08333333333334\\
270 68.83333333333333\\
280 69.25000000000001\\
290 64.08333333333334\\
300 65.66666666666667\\
310 65.25000000000001\\
320 67.91666666666666\\
330 72.41666666666666\\
340 70.0\\
350 70.08333333333334\\
360 64.91666666666667\\
370 73.66666666666666\\
380 73.66666666666667\\
390 72.0\\
400 68.58333333333333\\
410 74.16666666666666\\
420 70.33333333333334\\
430 71.58333333333333\\
440 72.5\\
450 74.41666666666669\\
460 74.0\\
470 72.33333333333334\\
480 73.41666666666666\\
490 74.0\\
500 74.25\\};

\addplot [line width=0.4mm, GCNRNIFull] %
table [row sep=\\]{%
0 50\\
10 51.08\\
20 50.67\\
30 52.33\\
40 52\\
50 51\\
60 52\\
70 53.67\\
80 53.33\\
90 52.25\\ 
100 57.33\\
110 55.58\\
120 55.92\\
130 61.67\\
140 62.58\\
150 65.25\\
160 62.5\\
170 68.25\\
180 67.33\\
190 68.33\\
200 68.08\\
210 70.67\\
220 71.58\\
230 68.83\\
240 70.58\\
250 68.33\\
260 71.67\\
270 69.58\\
280 70.42\\
290 70.75\\
300 72.42\\
310 71.83\\
320 72.17\\
330 71.50\\
340 72.50\\
350 73.50\\
360 71.17\\
370 72.00 \\
380 72.33 \\
390 73.58 \\
400 71.50 \\
410 73.92\\
420 72.83\\
430 71.33\\
440 72.08\\
450 74.00\\
460 74.58\\
470 74.58\\
480 74.58\\
490 75.25\\
500 73.75\\
};

\addplot [line width=0.4mm, GCNRNI50] %
table [row sep=\\]{%
0 50\\
10 51.17\\
20 51.17\\
30 54.33\\
40 53.42\\
50 52.25\\
60 51.58\\
70 55.08\\
80 56.75\\
90 55.08\\ 
100 56.08\\
110 57.25\\
120 59.00\\
130 60.33\\
140 58.67\\
150 61.00\\
160 66.75\\
170 66.00\\
180 67.67\\
190 68.83\\
200 72.58\\
210 77.58\\
220 84.08\\
230 87.33\\
240 89.58\\
250 92.75\\
260 94.41\\
270 96.00\\
280 95.92\\
290 97.50\\
300 96.92\\
310 97.17\\
320 97.83\\
330 97.83\\
340 98.75\\
350 99.00\\
360 98.92\\
370 98.83 \\
380 99.33 \\
390 99.25 \\
400 99.33 \\
410 98.92 \\
420 98.92 \\
430 99.00 \\
440 99.58 \\
450 99.58 \\
460 99.50 \\
470 99.25 \\
480 99.33 \\
490 97.42 \\
500 99.17 \\
};

\addplot [line width=0.4mm, GCNRNI87] %
table [row sep=\\]{%
0 50\\
10 51.83\\
20 50.83\\
30 50.83\\
40 51.08\\
50 52.08\\
60 50.25\\
70 53.17\\
80 51.75\\
90 53.83\\ 
100 55.58\\
110 53.50\\
120 56.17\\
130 61.42\\
140 67.00\\
150 73.50\\
160 77.58\\
170 83.00\\
180 85.67\\
190 90.42\\
200 96.50\\
210 98.08\\
220 99.00\\
230 99.17\\
240 99.08\\
250 99.08\\
260 99.42\\
270 99.42\\
280 99.67\\
290 99.67\\
300 99.33\\
310 99.33\\
320 99.25\\
330 99.50\\
340 99.50\\
350 99.50\\
360 99.42\\
370 99.33 \\
380 99.75 \\
390 99.67 \\
400 99.67 \\
410 99.08 \\
420 99.92 \\
430 99.83 \\
440 99.67 \\
450 99.50 \\
460 99.75 \\
470 99.17 \\
480 99.83 \\
490 99.67 \\
500 99.75 \\
};

\addplot [line width=0.4mm, GCNRNI12] %
table [row sep=\\]{%
0 50\\
10 51.33\\
20 50.42\\
30 51.33\\
40 53.83\\
50 52.58\\
60 52.08\\
70 56.08\\
80 54.50\\
90 56.50\\ 
100 56.00\\
110 58.33\\
120 58.08\\
130 62.17\\
140 65.41\\
150 70.25\\
160 76.67\\
170 81.75\\
180 85.73\\
190 89.67\\
200 92.08\\
210 94.33\\
220 92.83\\
230 95.83\\
240 95.67\\
250 95.75\\
260 97.00\\
270 97.67\\
280 97.67\\
290 95.17\\
300 97.67\\
310 98.25\\
320 98.17\\
330 99.33\\
340 99.33\\
350 99.33\\
360 99.08\\
370 99.25 \\
380 97.50 \\
390 99.33 \\
400 99.50 \\
410 99.58 \\
420 99.67 \\
430 99.75 \\
440 99.33 \\
450 96.42 \\
460 99.58 \\
470 99.58 \\
480 99.75 \\
490 99.42 \\
500 99.83 \\
};
\end{axis}
\end{tikzpicture}
		\caption{\CEXP.}
		\label{fig:modelConvergencePlus}
	\end{subfigure}%
	\begin{subfigure}{.32\textwidth}
		\centering
		\begin{tikzpicture}[scale=0.67]

\definecolor{ThreeGCN}{rgb}{0,0.516,0.66}
\definecolor{GCNStandard}{rgb}{0.98,0.79,0.26}
\definecolor{GCNRNI12}{rgb}{0,0.16,0.39}
\definecolor{GCNRNI50}{rgb}{0.40,0.66,0.26}
\definecolor{GCNRNI87}{rgb}{0.84,0.22,0.05}
\definecolor{GCNRNIFull}{rgb}{0.63,0.39,0.7}

\begin{axis}[
legend cell align={left},
legend columns=3,
legend entries={{GCN/C},{\GCNRNI/C}, {GCN-50\%/C},{GCN/E},{\GCNRNI/E}, {GCN-50\%/E}},
legend style={at={(-0.03, -0.21)}, font=\normalsize, line width=0.4mm, anchor=north west, draw=white!80.0!black},
label style={font=\large},
tick label style={font=\normalsize},
tick align=outside,
tick pos=left,
x grid style={white!69.01960784313725!black},
xlabel={Epoch},
xmajorgrids,
xmin=0, xmax=500, 
y grid style={white!69.01960784313725!black},
ymajorgrids,
ymin=49.0, ymax=100
]

\addplot [line width=0.4mm, GCNStandard, mark size=1pt]
table [row sep=\\]{%
10 50.0\\
20 50.5\\
30 52.0\\
40 51.33333333333334\\
50 53.666666666666664\\
60 63.33333333333332\\
70 60.66666666666667\\
80 66.83333333333333\\
90 70.5\\
100 71.83333333333333\\
110 72.0\\
120 71.16666666666667\\
130 74.00000000000001\\
140 70.16666666666667\\
150 65.16666666666666\\
160 75.33333333333333\\
170 59.0\\
180 75.83333333333333\\
190 69.5\\
200 79.33333333333333\\
210 79.83333333333333\\
220 87.00000000000001\\
230 84.66666666666667\\
240 80.66666666666666\\
250 77.83333333333333\\
260 82.16666666666667\\
270 87.66666666666667\\
280 88.5\\
290 78.16666666666666\\
300 81.33333333333333\\
310 80.5\\
320 85.83333333333331\\
330 94.83333333333333\\
340 90.0\\
350 90.16666666666667\\
360 79.83333333333334\\
370 97.33333333333333\\
380 97.33333333333334\\
390 94.0\\
400 87.16666666666667\\
410 98.33333333333331\\
420 90.66666666666666\\
430 93.16666666666666\\
440 95.0\\
450 98.83333333333333\\
460 98.00000000000001\\
470 94.66666666666667\\
480 96.83333333333334\\
490 98.00000000000001\\
500 98.5\\
};

\addplot [line width=0.4mm, GCNRNIFull,mark size=1pt]
table [row sep=\\]{%
10 51.33333333333333\\
20 52.16666666666667\\
30 51.83333333333333\\
40 51.33333333333333\\
50 53.49999999999999\\
60 51.33333333333334\\
70 53.333333333333336\\
80 53.16666666666667\\
90 52.83333333333334\\
100 52.5\\
110 53.166666666666664\\
120 54.0\\
130 57.33333333333333\\
140 59.33333333333333\\
150 53.333333333333336\\
160 53.0\\
170 58.33333333333333\\
180 56.49999999999999\\
190 52.16666666666667\\
200 53.99999999999999\\
210 52.0\\
220 56.99999999999999\\
230 56.33333333333332\\
240 53.166666666666664\\
250 54.0\\
260 54.50000000000001\\
270 51.16666666666667\\
280 53.166666666666664\\
290 54.50000000000001\\
300 53.0\\
310 52.5\\
320 53.99999999999999\\
330 54.666666666666664\\
340 53.33333333333334\\
350 56.333333333333336\\
360 55.00000000000001\\
370 52.33333333333333\\
380 54.833333333333336\\
390 54.333333333333336\\
400 55.833333333333336\\
410 55.49999999999999\\
420 54.50000000000001\\
430 54.16666666666667\\
440 56.50000000000001\\
450 55.16666666666667\\
460 57.49999999999999\\
470 53.833333333333336\\
480 57.49999999999999\\
490 54.83333333333334\\
500 57.666666666666664\\
};

\addplot [line width=0.4mm, GCNRNI50,mark size=1pt] %
table [row sep=\\]{%
10 51.16666666666667\\
20 51.5\\
30 53.666666666666664\\
40 50.66666666666666\\
50 52.166666666666664\\
60 51.16666666666667\\
70 54.49999999999999\\
80 53.83333333333334\\
90 53.49999999999999\\
100 53.833333333333336\\
110 54.33333333333332\\
120 55.666666666666664\\
130 57.333333333333336\\
140 58.166666666666664\\
150 57.00000000000001\\
160 59.5\\
170 58.50000000000001\\
180 64.0\\
190 63.16666666666666\\
200 65.83333333333333\\
210 73.16666666666667\\
220 81.33333333333333\\
230 86.16666666666667\\
240 89.16666666666666\\
250 93.66666666666667\\
260 93.66666666666666\\
270 94.33333333333334\\
280 94.16666666666667\\
290 95.83333333333333\\
300 95.33333333333334\\
310 97.16666666666667\\
320 96.83333333333331\\
330 98.00000000000001\\
340 98.66666666666667\\
350 98.16666666666667\\
360 98.83333333333333\\
370 97.83333333333331\\
380 98.66666666666667\\
390 99.33333333333334\\
400 99.16666666666669\\
410 99.16666666666669\\
420 98.83333333333333\\
430 99.16666666666666\\
440 99.16666666666666\\
450 99.49999999999999\\
460 99.33333333333334\\
470 98.83333333333333\\
480 99.0\\
490 97.33333333333334\\
500 99.33333333333334\\
};

\addplot [dotted, line width=0.6mm, GCNStandard, mark size=1pt]
table [row sep=\\]{%
0 50\\
10 50\\
20 50\\
30 50\\
40 50\\
50 50\\
60 50\\
70 50\\
80 50\\
90 50\\
100 50\\
110 50\\
120 50\\
130 50\\
140 50\\
150 50\\
160 50\\
170 50\\
180 50\\
190 50\\
200 50\\
210 50\\
220 50\\
230 50\\
240 50\\
250 50\\
260 50\\
270 50\\
280 50\\
290 50\\
300 50\\
310 50\\
320 50\\
330 50\\
340 50\\
350 50\\
360 50\\
370 50\\
380 50\\
390 50\\
400 50\\
410 50\\
420 50\\
430 50\\
440 50\\
450 50\\
460 50\\
470 50\\
480 50\\
490 50\\
500 50\\
};

\addplot [dotted, line width=0.6mm, GCNRNIFull, mark size=1pt]
table [row sep=\\]{%
10 50.66666666666666\\
20 51.0\\
30 50.66666666666666\\
40 52.5\\
50 52.66666666666667\\
60 55.666666666666664\\
70 51.16666666666666\\
80 57.16666666666666\\
90 56.49999999999999\\
100 56.99999999999999\\
110 56.666666666666664\\
120 61.5\\
130 67.83333333333333\\
140 69.16666666666667\\
150 72.16666666666669\\
160 76.16666666666666\\
170 79.5\\
180 78.16666666666666\\
190 79.83333333333333\\
200 81.66666666666667\\
210 82.99999999999999\\
220 85.0\\
230 84.00000000000001\\
240 85.83333333333334\\
250 82.33333333333333\\
260 87.16666666666667\\
270 88.33333333333334\\
280 88.16666666666666\\
290 88.83333333333333\\
300 88.83333333333333\\
310 88.66666666666664\\
320 87.16666666666667\\
330 89.0\\
340 86.33333333333333\\
350 91.83333333333333\\
360 89.66666666666666\\
370 92.66666666666666\\
380 90.16666666666666\\
390 90.66666666666666\\
400 90.0\\
410 90.16666666666666\\
420 92.83333333333333\\
430 91.16666666666666\\
440 92.83333333333333\\
450 90.99999999999999\\
460 90.99999999999999\\
470 93.33333333333333\\
480 92.83333333333331\\
490 92.16666666666666\\
500 90.66666666666666\\
};

\addplot [dotted, line width=0.6mm, GCNRNI50, mark size=1pt]
table [row sep=\\]{%
10 50.0\\
20 53.0\\
30 55.666666666666664\\
40 53.666666666666664\\
50 54.33333333333332\\
60 54.0\\
70 53.66666666666667\\
80 55.99999999999999\\
90 58.5\\
100 56.833333333333336\\
110 59.0\\
120 62.83333333333333\\
130 63.16666666666666\\
140 62.66666666666667\\
150 64.5\\
160 68.5\\
170 67.5\\
180 73.66666666666667\\
190 74.33333333333333\\
200 76.66666666666667\\
210 82.66666666666666\\
220 90.16666666666666\\
230 90.83333333333331\\
240 90.33333333333333\\
250 93.83333333333331\\
260 95.83333333333333\\
270 96.0\\
280 97.33333333333333\\
290 98.33333333333331\\
300 98.83333333333333\\
310 97.49999999999999\\
320 97.66666666666666\\
330 98.16666666666667\\
340 98.83333333333333\\
350 99.66666666666667\\
360 100.0\\
370 99.49999999999999\\
380 100.0\\
390 99.66666666666667\\
400 99.33333333333334\\
410 99.66666666666667\\
420 99.33333333333334\\
430 99.0\\
440 99.66666666666667\\
450 99.66666666666667\\
460 100.0\\
470 99.49999999999999\\
480 99.66666666666667\\
490 97.66666666666666\\
500 99.16666666666669\\
};

\end{axis}
\end{tikzpicture}
		\caption{\EXPTwo (/E) and \Corrupt (/C).}
		\label{fig:modelConvergenceSplit}
	\end{subfigure}%
	\label{fig:res}
	\caption{Learning curves across all experiments for all models.}
\end{figure*}
Full test accuracy results for all models are reported in Table \ref{tab:exp1Results}, and model convergence for \ThreeGNNFull and all \GCNRNI~models are shown in \Cref{fig:modelConvergence}.
In line with \Cref{theo:uni}, \GCNRNI~achieves a near-perfect performance on \EXP, substantially surpassing 50\%. Indeed, \GCNRNI~models achieve above 95\% accuracy with all four RNI distributions. %
This finding further supports observations made with rGNNs \cite{SatoRandom2020}, and shows that RNI is also beneficial in settings beyond structure detection. 
Empirically, we observed that \GCNRNI~is highly sensitive to changes in learning rate, activation function, and/or randomization distribution, and required delicate tuning to achieve its best performance.

Surprisingly, PPGN does not achieve a performance above 50\%, despite
being theoretically 2-WL expressive.  
Essentially, PPGN learns an approximation of 2-WL, based on power-sum multi-symmetric polynomials (PMP), but fails to distinguish \EXP graph pairs, despite extensive training. This suggests that PPGN struggles to learn the required PMPs, and we could not improve these results, both for training and testing, with hyperparameter tuning. 
Furthermore, PPGN requires exponentially many data
samples in the size of the input graph \cite{PunyLRGA} for learning. Hence, PPGN is likely struggling to discern between \EXP graph pairs due to the smaller sample size and variability of the dataset. 
\OneTwoThreeLocal also only achieves 50\% accuracy, which can be attributed to theoretical model limitations. Indeed, this algorithm only considers 3-tuples of nodes that form a connected subgraph, thus discarding disconnected 3-tuples, where the difference between \EXP cores lies. This further highlights the difficulty of \EXP, as even relaxing 2-WL reduces the model to random performance.
Note that \ThreeGNNFull achieves near-perfect performance, as it explicitly has the necessary theoretical power, irrespective of learning constraints, and must only learn appropriate injective aggregation functions for neighbor aggregation \cite{Keyulu18}.

In terms of convergence, we observe that \ThreeGNNFull converges
significantly faster than \GCNRNI~models, for all randomization
percentages. Indeed, \ThreeGNNFull only requires about 10 epochs to
achieve optimal performance, whereas \GCNRNI~models all require over 100 epochs. 
Intuitively, this slower convergence of \GCNRNI~can be attributed to a harder learning task compared to \ThreeGNNFull: Whereas
\ThreeGNNFull learns from deterministic embeddings, and can naturally discern between dataset cores, \GCNRNI~relies on RNI to discern between \EXP data points, via an artificial node ordering. This implies that \GCNRNI~must leverage RNI to detect structure, then subsequently learn robustness against RNI variability, which makes its learning task especially challenging.

Our findings suggest that RNI practically improves  MPNN expressiveness, and makes them competitive with higher-order models, despite being less demanding computationally. Indeed, for a 50-node graph, GCN-RNI only requires 3200 parameters (using 64-dimensional embeddings), whereas \ThreeGNNFull requires 1,254,400 parameters. Nonetheless, GCN-RNI performs comparably to \ThreeGNNFull, and, unlike the latter, can easily scale to larger instances. This increase in expressive power, however, comes at the cost of slower convergence. Even so, RNI proves to be a promising direction for building scalable yet powerful MPNNs. 
\subsection{How Does RNI Behave on Variable Data?} 

In the earlier experiment, RNI practically improves the expressive power of GCNs over \EXP. However, \EXP solely evaluates expressiveness, and this leaves multiple questions open: How does RNI impact learning when data contains instances with varying expressiveness requirements, and how does RNI affect generalization on more variable datasets? We experiment with \CEXP to explicitly address these questions.

We generated \CEXP by generating another 600 graph pairs, then selecting 300 of these and modifying their satisfiable graph, yielding \Corrupt.
\CEXP is well-suited for holistically evaluating the efficacy of RNI, as it evaluates the contribution of RNI on \EXPTwo conjointly with a second learning task on \Corrupt involving very similar core structures, and assesses the effect of different randomization degrees on overall and subset-specific model performance. 

In this experiment, we train \GCNRNI~(with varying randomization degrees) and \ThreeGNNFull on \CEXP, and compare their accuracy. For \GCNRNI, we observe the effect of RNI on learning \EXPTwo and \Corrupt, and the interplay between these tasks. %
In all experiments, we use the normal distribution for RNI, given its strong performance in the earlier experiment.

The learning curves of all \GCNRNI~ and \ThreeGNNFull on \CEXP are shown in \Cref{fig:modelConvergencePlus}, and the same curves for the \EXPTwo and \Corrupt subsets are shown in \Cref{fig:modelConvergenceSplit}. As on \EXP, \ThreeGNNFull converges very quickly, exceeding 90\% test accuracy within 25 epochs on \CEXP. By contrast, \GCNRNI, for all randomization levels, converges much slower, around after 200 epochs, despite the small size of input graphs ($\sim$70 nodes at most). Furthermore, fully randomized \GCNRNI~ performs worse than partly randomized \GCNRNI, particularly on \CEXP, due to its weak performance on \Corrupt.

First, we observe that partial randomization significantly improves performance. This can clearly be seen on \CEXP, where GCN-12.5\%RNI and GCN-87.5\%RNI achieve the best performance, by far outperforming GCN-RNI, which struggles on \Corrupt.
This can be attributed to having a better inductive bias than a fully randomized model.  Indeed, GCN-12.5\%RNI has mostly deterministic node
embeddings, which simplifies learning over \Corrupt. This also applies to GCN-87.5\%RNI, where the number of deterministic dimensions, though small, remains sufficient.  Both models also benefit from randomization for \EXPTwo, similarly to a fully randomized GCN.
GCN-12.5\%RNI and GCN-87.5\%RNI effectively achieve the best of both worlds on \CEXP, leveraging inductive bias from deterministic node embeddings, while harnessing the power of RNI to perform strongly on \EXPTwo. This is best shown in \Cref{fig:modelConvergenceSplit}, where standard GCN fails to learn \EXPTwo, fully randomized GCN-RNI struggles to learn \Corrupt, and the semi-randomized GCN-50\%RNI achieves perfect performance on both subsets. We also note that partial RNI, when applied to several real datasets, where 1-WL power is sufficient, did not harm performance \cite{SatoRandom2020}, and thus at least preserves the original learning ability of MPNNs in such settings. Overall, these are surprising findings, which suggest that MPNNs can viably improve across all possible data with partial and even small amounts of randomization. 

Second, we  observe that the fully randomized \GCNRNI~ performs substantially worse than its partially randomized counterparts. Whereas fully randomized \GCNRNI~ only performs marginally worse on \EXP (cf. \Cref{fig:modelConvergence}) than partially randomized models, this gap is very large on \CEXP, primarily due to \Corrupt. 
This observation concurs with the earlier idea of inductive bias: Fully randomized \GCNRNI~ loses all node type information, which is key for \Corrupt, and therefore struggles. Indeed, the model fails to achieve even 60\% accuracy on \Corrupt, where other models are near perfect, and also relatively struggles on \EXPTwo, only reaching 91\% accuracy and converging slower. 

Third, all \GCNRNI~models, at all randomization levels, converge significantly slower than \ThreeGNNFull on both \CEXP and \EXP. However, an interesting phenomenon can be seen on \CEXP: All \GCNRNI ~models fluctuate around 55\% accuracy within the first 100 epochs, suggesting a struggle jointly fitting both \Corrupt and \EXPTwo, before they ultimately improve. This, however, is not observed with \ThreeGNNFull. Unlike on \EXP, randomness is not necessarily beneficial on \CEXP, as it can hurt performance on \Corrupt.
Hence, RNI-enhanced models must additionally learn to isolate  deterministic dimensions for \Corrupt, and randomized dimensions for \EXPTwo. 
These findings consolidate the earlier observations made on \EXP, and highlight that the variability and slower learning for RNI also hinges on the complexity of the input dataset.  

Finally, we observe that both fully randomized \GCNRNI, and, surprisingly, \OneGCN, struggle to learn \Corrupt relative to partially randomized \GCNRNI. We also observe that \OneGCN does not ``struggle'', and begins improving consistently from the start of training. 
These observations can be attributed to key conceptual , but very distinct hindrances impeding both models. For \OneGCN, the model is jointly trying to learn both \EXPTwo and \Corrupt, when it provably cannot fit the former. This joint optimization severely hinders \Corrupt learning, as data pairs from both subsets are highly similar, and share identically generated UNSAT graphs. 
Hence, \OneGCN, in attempting to fit SAT graphs from both subsets, knowing it cannot distinguish \EXPTwo pairs, struggles to learn the simpler difference in \Corrupt pairs. For \GCNRNI, the model discards key type information, so must only rely on structural differences to learn \Corrupt, which impedes its convergence. All in all, this further consolidates the promise of partial RNI as a means to combine the strengths of both deterministic and random features.

\section{Related Work}
\label{sec:rw}

MPNNs have been enhanced with RNI \cite{SatoRandom2020}, such that the model trains and runs with partially randomized initial node features. These models, denoted rGNNs, are shown to near-optimally approximate solutions to specific combinatorial optimization problems, and can distinguish between 1-WL indistinguishable graph pairs based on fixed local substructures. Nonetheless, the precise impact of RNI on GNNs for learning arbitrary functions over graphs remained open. Indeed, rGNNs are only shown to admit parameters that can detect a \emph{unique, fixed} substructure, and thus tasks requiring \emph{simultaneous} detection of multiple combinations of structures, as well as problems having no locality or structural biases, are not captured by the existing theory. 

Our work improves on Theorem 1 of \cite{SatoRandom2020}, and shows \emph{universality} of MPNNs with RNI. Thus, it shows that arbitrary real-valued functions over graphs can be learned by MPNNs with RNI. Our result is distinctively based on a logical characterization of MPNNs, which allows us to link the size of the MPNN with the descriptive complexity of the target function to be learned.
Empirically, we highlight that the power of RNI in a significantly more challenging setting, using a target function (\SAT) which does not rely on local structures, is \emph{hard} to approximate. %

Similarly to RNI, random pre-set color features have been used to disambiguate between nodes \cite{DasoulasSSV20}. This approach, known as CLIP, introduces randomness to node representations, but explicitly makes graphs distinguishable by construction. By contrast, we study random features produced by RNI, which (i)~are not designed a priori to distinguish nodes, (ii)~do not explicitly introduce a fixed underlying structure, and (iii)~yield potentially infinitely many representations for a single graph. In this more general setting, we nonetheless show that RNI adds expressive power to distinguish nodes with high probability, leads to a universality result, and performs strongly in challenging problem settings. 

\section{Summary and Outlook}
We studied the expressive power of MPNNs with RNI, and showed that these models are universal and  preserve MPNN invariance in expectation. 
We also empirically evaluated these models on carefully designed datasets, and observed that RNI improves their learning ability, but slows their convergence. 
Our work delivers a theoretical result, supported by practical insights, to quantify the effect of RNI on GNNs. An interesting topic for future work is to study whether polynomial functions can be captured via efficient constructions; see, e.g., \cite{GroheLogicGNN} for related open problems.

\section*{Acknowledgments}
This work was supported by the Alan Turing Institute under the UK EPSRC grant EP/N510129/1, by the AXA Research Fund, and by the EPSRC grants EP/R013667/1 and EP/M025268/1. Ralph Abboud is funded by the Oxford-DeepMind Graduate Scholarship and the Alun Hughes Graduate Scholarship. Experiments  were conducted on the Advanced Research Computing (ARC) cluster administered by the University of Oxford.

\bibliographystyle{named}
{\small
\bibliography{main}
}
\appendix
\section{Appendix}
\label{sec:app}

\subsection{Propositional Logic}
We briefly present propositional logic, which underpins the dataset generation. Let $S$ be a (finite) set $S$ of propositional variables. A \emph{literal} is defined as $v$, or $\bar{v}$ (resp., $\neg v$), where $v \in S$. A disjunction of literals is a \emph{clause}. The \emph{width} of a clause is defined as the number of literals it contains. 
A formula $\phi$ is in \emph{conjunctive normal form (\CNF)} if it is a conjunction of clauses. A  \CNF has width $k$ if it contains clauses of width at most $k$, and is referred to as a $k-$\CNF. To illustrate, the formula ${\phi=(x_1 \lor \neg x_3) \land (x_4 \lor x_1)}$ is a \CNF with clauses of width 2.

An assignment $\nu: S \mapsto \{0,1\}$ maps variables to False (0), or True (1), and \emph{satisfies} $\phi$, which we denote by $\nu\models \phi$, in the usual sense, where $\models$ is propositional entailment. 
Given a propositional formula $\phi$, the \emph{satisfiability problem}, commonly known as \SAT, consists of determining whether $\phi$ admits a satisfying assignment, and is NP-complete \cite{cook1971complexity}. 

\subsection{Proof of \Cref{theo:uni}}

We first prove a Boolean version of the theorem.

\begin{lemma}\label{lem:boolean}
  Let $n\ge 1$, and let $f:\CG_n\to\{0,1\}$ be an
  invariant
  Boolean
  function. Then, for all $\epsilon,\delta>0$ there is a MPNN with RNI
  that $(\epsilon,\delta)$-approximates $f$.
\end{lemma}

To prove this lemma, we use a logical characterization of the
expressiveness of MPNNs, which we always assume to admit global
readouts. Let $\LC$ be the extension of first-order predicate logic using
counting quantifiers of the form $\exists^{\ge k}x$ for $k\ge 0$,
where $\exists^{\ge k}x\phi(x)$ means that there are at least $k$
elements $x$ satisfying $\phi$.

For example, consider the formula
\begin{equation}
  \label{eq:a1}
  \phi(x):=\neg\exists^{\ge 3}y\big(
  E(x,y)\wedge \exists^{\ge 5} z E(y,z)\big).
\end{equation}
This is a formula in the language of graphs; $E(x,y)$ means that there
is an edge between the nodes interpreting $x$ and $y$.
For a graph $G$ and a vertex $v\in V(G)$, we have
$G\models\phi(v)$ (``$G$ satisfies $\phi$ if the variable $x$ is
interpreted by the vertex $v$") if and only if $v$ has at most $2$
neighbors in $G$ that have degree at least $5$.

We will not only
consider formulas in the language of graphs, but also formulas in
the \emph{language of colored graphs}, where in addition to the binary
edge relation we also have unary relations, that is, sets of
nodes, which we may view as colors of the nodes. For example, the
formula
\[
  \psi(x):=\exists^{\ge 4}y\big(E(x,y)\wedge \textit{RED}(y)\big)
\]
says that node $x$ has at least $4$ red neighbors (more precisely,
neighbors in the unary relation $\textit{RED}$). Formally, we assume
we have fixed infinite list $R_1,R_2,\ldots$ of color symbols that we
may use in our formulas. Then a \emph{colored graph} is a graph together
with a mapping that assigns a finite set $\rho(v)$ of colors $R_i$ to
each vertex (so we allow one vertex to have more than one, but only
finitely many, colors).

A \emph{sentence} (of the logic $\LC$ or any other logic) is a formula
without free variable. Thus a sentence expresses a property of a
graph, which we can also view as a Boolean function. For a sentence
$\phi$ we denote this function by $\llbracket\phi\rrbracket$. If
$\phi$ is a sentence in the language of (colored) graphs, then for
every (colored) graph
$G$ we have  $\llbracket\phi\rrbracket(G)=1$ if $G\models\phi$ and
$\llbracket\phi\rrbracket(G)=0$ otherwise.

It is easy to see that $\LC$ is only a syntactic extension of first
order logic $\FO$---for every $\LC$-formula there is a logically
equivalent $\FO$-formula. To see this, note that we can simulate   
$\exists^{\ge k}x$ by $k$ ordinary existential quantifiers:
$\exists^{\ge k}x$ is equivalent to $\exists x_1\ldots\exists
x_k\Big(\bigwedge_{1\le i<j\le k}x_i\neq x_j\wedge\bigwedge_{1\le i\le
  k}\phi(x_i)\Big)$. However, counting quantifiers add expressiveness if we restrict the number of
variables. The $\LC$-formula $\exists^{\ge k}x\;(x=x)$ (saying that there
are at least $k$ vertices) with just one variable is not equivalent to any $\FO$-formula
using less than $k$ variables. By $\LC^k$ we denote the fragment of
$\LC$ consisting of all formulas with at most $k$ variables.

For example, the formula $\phi(x)$ in (\ref{eq:a1}) is in $\LC^3$, but
not in $\LC^2$. But $\phi(x)$ is equivalent to the following formula
$\phi'(x)$ in $\LC^2$:
\begin{equation}
  \label{eq:a2}
  \phi'(x):=\neg\exists^{\ge 3}y\big(
  E(x,y)\wedge \exists^{\ge 5} x E(y,x)\big).
\end{equation}
The fragments $\LC^k$ are interesting for us, because their
expressiveness corresponds to that of $(k-1)$-WL and hence to that of
$k$-GNNs. More precisely, for all $k\ge 2$, two graphs $G$ and $H$
satisfy the same $\LC^{k}$-sentences if and only if $(k-1)$-WL does
not distinguish them \cite{CaiFI92}. By the results of \cite{MorrisAAAI19,Keyulu18} this implies,
in particular, that two graphs are indistinguishable by all MPNNs if
and only if they satisfy the same $\LC^2$-sentences.
\cite{BarceloKM0RS20} strengthened this result and showed that every
$\LC^2$-sentence can be simulated by an MPNN. 

\begin{lemma}[\cite{BarceloKM0RS20}]\label{lem:barcelo}
  For every $\LC^2$-sentence $\phi$ and every $\epsilon>0$ there is an
  MPNN that $\epsilon$-approximates $\llbracket\phi\rrbracket$.
\end{lemma}

Since here we are talking about deterministic MPNNs, there is no
randomness involved, and we just say \emph{``$\epsilon$-approximates''}
instead of ``$(\epsilon,1)$-approximates''.

Lemma~\ref{lem:barcelo} not only holds for sentences in the language
of graphs, but also for sentences in the language of colored
graphs. Let us briefly discuss the way MPNNs access such colors. We
encode the colors using one-hot vectors that are part of the
initial states of the nodes. For example, if we have a formula that
uses color symbols among $R_1,\ldots,R_k$, then we reserve $k$
places in the initial state $\vec x_v=(x_{v1},\ldots,x_{v\ell})$ of each
vertex $v$ (say, for convenience, $x_{v1},\ldots,x_{vk}$) and we
initialize $\vec x_v$ by letting $x_{vi}=1$ if $v$ is in $R_i$ and
$x_{vi}=0$ otherwise.

Let us call a colored graph $G$ \emph{individualized} if for any two
distinct vertices $v,w\in V(G)$ the sets $\rho(v),\rho(w)$ of colors
they have are distinct.
Let us say that a sentence $\chi$ \emph{identifies} a (colored) graph $G$ if for
all (colored) graphs $H$ we have $H\models\chi$ if and only if $H$ is
isomorphic to $G$.

\begin{lemma}
  For every individualized colored graph $G$ there is a
  $\LC^2$-sentence $\chi_G$ that identifies $G$. 
\end{lemma}

\begin{proof}
  Let $G$ be an individualized graph. For every vertex $v\in V(G)$,
  let
  \[
    \alpha_v(x):=\bigwedge_{R\in\rho(v)}R(x)\wedge\bigwedge_{R\in\{R_1,\ldots,R_k\}\setminus\rho(x)}\neg
    R(x).
  \]
  Then $v$ is the unique vertex of $G$ such that
  $G\models\alpha_v(v)$. For every pair $v,w\in V(G)$ of
  vertices, we let
  \[
    \beta_{vw}(x,y):=
    \begin{cases}
      \alpha_v(x)\wedge\alpha_w(y)\wedge E(x,y)&\text{if }(v,w)\in
      E(G),\\
      \alpha_v(x)\wedge\alpha_w(y)\wedge \neg E(x,y)&\text{if }(v,w)\not\in
      E(G).
    \end{cases}
  \]
  We let
  \begin{align*}
    \chi_G:=&\bigwedge_{v\in V(G)}\big(\exists
    x\alpha_v(x)\wedge\neg\exists^{\ge
      2}x\alpha_v(x)\big)~\wedge \\
      &\bigwedge_{v,w\in V(G)}\exists x\exists y\beta_{vw}(x,y).
  \end{align*}
  It is easy to see that $\chi_G$ identifies $G$.
\end{proof}

For $n,k\in\Nat$, we let $\CG_{n,k}$ be the class of
all individualized colored graphs that only use colors among $R_1,\ldots,R_k$.

\begin{lemma}
  Let $h:\CG_{n,k}\to\{0,1\}$ be an invariant Boolean function. Then
  there exists a $\LC^2$-sentence $\psi_h$ such that for all
  $G\in\CG_{n,k}$ it holds that $\llbracket\psi_h\rrbracket(G)=h(G)$.
\end{lemma}

\begin{proof}
  Let $\mathcal H\subseteq\CG_{n,k}$ be the subset consisting of all graphs
  $H$ with $h(H)=1$. We let
  \[
    \psi_h:=\bigvee_{H\in\mathcal H}\chi_H.
  \]
  We eliminate duplicates in the disjunction. Since up to isomorphism,
  the class $\CG_{n,k}$ is finite, this makes the disjunction finite and
  hence $\psi_h$ well-defined.
\end{proof}

The \emph{restriction} of a colored graph $G$ is the underlying plain
graph, that is, the graph $G^\vee$ obtained from the colored graph $G$ by
forgetting all the colors. Conversely, a colored graph $G^\wedge$ is an
\emph{expansion} of a plain graph $G$ if $G=(G^\wedge)^\vee$. %

\begin{corollary}\label{cor:inv}
  Let $f:\CG_{n}\to\{0,1\}$ be an invariant Boolean function. Then
  there exists a $\LC^2$-sentence $\phi^\wedge_f$ (in the language
  of colored graphs) such that for all $G\in\CG_{n,k}$ it holds that
  $\llbracket\psi^\wedge_f\rrbracket(G)=f(G^\vee)$.
\end{corollary}

Towards proving Lemma~\ref{lem:boolean}, we fix an $n\ge 1$ and a
$\epsilon,\delta>0$. We let 
\[
  c:=\left\lceil\frac{2}{\delta}\right\rceil
  \quad
  \text{and}
  \quad
  k:=c^2\cdot n^3
\]
The technical
details of the proof of Lemma~\ref{lem:boolean} and
Theorem~\ref{theo:uni} depend on the exact choice of the random
initialization and the activation functions used in the neural
networks, but the idea is always the same. For simplicity, we assume
that we initialize the states $\vec x_v=(x_{v1},\ldots,x_{v\ell})$ of all
vertices to $(r_v,0,\ldots,0)$, where $r_v$ for $v\in V(G)$ are chosen
independently uniformly at random from $[0,1]$. As our activation
function $\sigma$, we choose the linearized sigmoid function defined
by $\sigma(x)=0$ for $x<0$, $\sigma(x)=x$ for $0\le x<1$, and
$\sigma(x)=1$ for $x\ge 1$.

\begin{lemma}\label{lem:prob}
   Let $r_1,\ldots,r_n$ be chosen
   independently uniformly at random from the interval $[0,1]$.
   For $1\le i\le n$  and $1\le j\le c\cdot n^2$, let
   \[
     s_{ij}:=k\cdot r_i-(j-1)\cdot\frac{k}{c\cdot n^2}.
   \]
   Then with probability greater than $1-\delta$, the following
   conditions are satisfied.
   \begin{enumerate}
   \item[(i)] For all $i\in\{1,\ldots,n\},j\in\{1,\ldots,c\cdot n^2\}$ it
     holds that $\sigma(s_{ij})\in\{0,1\}$.
   \item[(ii)] For all distinct $i,i'\in\{1,\ldots,n\}$ there exists a
     $j\in\{1,\ldots,c\cdot n^2\}$ such that
     $\sigma(s_{ij})\neq\sigma(s_{i'j})$.
   \end{enumerate}
\end{lemma}

\begin{proof}
  For every $i$, let $p_i:=\lfloor r_i\cdot k\rfloor$. Since $k\cdot
  r_i$ is uniformly random from the interval $[0,k]$, the integer
  $p_i$ is uniformly random from $\{0,\ldots,k-1\}$. Observe that
  $0<\sigma(s_{ij})<1$ only if $p_i- (j-1)\cdot\frac{k}{c\cdot n^2}=0$ (here
  we use the fact that $k$ is divisible by $c\cdot n^2$). The probability that
  this happens is $\frac{1}{k}$. Thus, by the Union Bound,
  \begin{equation}
    \label{eq:a3}
    \Pr\big(\exists i,j:\;0<\sigma(s_{ij})<1\big)\le\frac{c\cdot n^3}{k}.
  \end{equation}
  Now let $i,i'$ be distinct and suppose that
  $\sigma(s_{ij})=\sigma(s_{i'j})$ for all $j$. Then for all $j$ we
  have $s_{ij}\le 0\iff s_{i'j}\le 0$ and therefore $\lfloor
  s_{ij}\rfloor\le 0\iff \lfloor
  s_{i'j}\rfloor\le 0$. This implies that
  \begin{align}
  \begin{split}
    &\forall j\in\{1,\ldots,c\cdot n^2\}:\\
    &\quad p_i\le
    (j-1)\cdot\frac{k}{c\cdot n^2}\iff  p_{i'}\le
    (j-1)\cdot\frac{k}{c\cdot n^2}. \label{eq:a4}
    \end{split}
  \end{align}
  Let $j^*\in\{1,\ldots, c\cdot n^2\}$ such that
  \[p_i\in\Big\{(j^*-1)\cdot\frac{k}{c\cdot
    n^2},\ldots,j^*\cdot\frac{k}{c\cdot n^2}-1\Big\}.\] Then
  by (\ref{eq:a4}) we have: \[p_i'\in
  \Big\{(j^*-1)\cdot\frac{k}{c\cdot n^2},\ldots,j^*\cdot\frac{k}{c\cdot n^2}-1\Big\}.\] As
  $p_{i'}$ is independent of $p_i$ and hence of $j^*$, the probability
  that this happens is at most $\frac{1}{k}\cdot
  \frac{k}{c\cdot n^2}=\frac{1}{c\cdot n^2}$.
  This proves that for all distinct $i,i'$ the probability that 
  $\sigma(s_{ij})=\sigma(s_{i'j})$ is at most $\frac{1}{c\cdot
    n^2}$. Hence, again by the Union Bound,
  \begin{equation}
    \label{eq:a5}
    \Pr\big(\exists i\neq i'\forall
    j:\;\sigma(s_{ij})=\sigma(s_{i'j})\big)\le \frac{1}{c}.
  \end{equation}
  (\ref{eq:a3}) and (\ref{eq:a5}) imply that the probability that either (i) or (ii) is
  violated is at most
  \[
    \frac{c\cdot n^3}{k}+\frac{1}{c}\le \frac{2}{c}\le\delta.
    \qedhere
  \]
\end{proof}

\begin{proof}[Proof of Lemma~\ref{lem:boolean}]
   For given function $f:\CG_n\to\{0,1\}$, we choose the sentence
   $\psi^\wedge_f$ according to Corollary~\ref{cor:inv}. Applying
   Lemma~\ref{lem:barcelo} to this sentence and $\epsilon$, we obtain
   an MPNN $\CN_f$ that on a colored graph $G\in \CG_{n,k}$ computes
   an $\epsilon$-approximation of $f(G^\vee)$.
  
   Without loss of generality, we assume that the vertex set of the
   input graph to our MPNN is $\{1,\ldots,n\}$. We choose $\ell$ (the
   dimension of the state vectors) in such a way that
   $\ell\ge c\cdot n^2$ and $\ell$ is at least as large as the
   dimension of the state vectors of $\CN_f$. Recall that
   the state vectors are initialized as $\vec
   x^{(0)}_i=(r_i,0,\ldots,0)$ for values $r_i$ chosen independently
   uniformly at random from the interval $[0,1]$.

   In the first step, our MPNN computes the purely local
   transformation (no messages need to be passed) that maps
   $\vec x^{(0)}_i$ to
   $\vec x^{(1)}_i=(x^{(1)}_{i1},\ldots,x^{(1)}_{i\ell})$ with
   \[
     x^{(1)}_{ij}=
     \begin{cases}
       \sigma\Big(k\cdot r_i-(j-1)\cdot\frac{k}{c\cdot n^2}\Big)&\text{for }1\le j\le
       c\cdot n^2,\\
       0&\text{for }c\cdot n^2+1\le j\le\ell.
     \end{cases}
   \]
   Since we treat $k,c,n$ as constants, the mapping \[r_i\mapsto k\cdot
   r_i-(j-1)\cdot\frac{k}{c\cdot n^2}\] is just a linear mapping
   applied to $r_i=x^{(0)}_{i1}$.

   By Lemma~\ref{lem:prob}, with probability at least $1-\delta$, the
   vectors $\vec x_i^{(1)}$ are mutually distinct $\{0,1\}$-vectors,
   which we view as encoding a coloring of the input graph with colors
   from $R_1,\ldots,R_k$. Let $G^\wedge$ be the resulting colored
   graph. Since the vectors $\vec x^{(0)}_i$ are
   mutually distinct, $G^\wedge$ is individualized and thus in the
   class $\CG_{n,k}$. We now apply the MPNN $\CN_f$, and it computes a
   value $\epsilon$-close to \[\llbracket\psi_f^\wedge\rrbracket(G^\wedge)=f((G^\wedge)^\vee)=f(G).\]
 \end{proof}

 \begin{proof}[Proof of Theorem~\ref{theo:uni}]
   Let $f:\CG_n\to \Real$ be invariant. Since $\CG_n$ is finite, the range
   $Y:=f(\CG_n)$ is finite. To be precise, we have
   $N:=|Y|\le|\CG_n|=2^{\binom{n}{2}}$. 

   Say, $Y=\{y_1,\ldots,y_N\}$. For
   $i=1,\ldots,N$, let $g_i:\CG_n\to\{0,1\}$ be the Boolean function
   defined by
   \[
     g_i(G)=
     \begin{cases}
       1&\text{if }f(G)=y_i,\\
       0&\text{otherwise}.
     \end{cases}
   \]
   Note that $g_i$ is invariant. Let $\epsilon,\delta>0$ and $\epsilon':=\frac{\epsilon}{\max Y}$
   and $\delta':=\frac{\delta}{N}$. By Lemma~\ref{lem:boolean}, for
   every $i\in\{1,\ldots,N\}$ there is an MPNN with RNI $\CN_i$ that
   $(\epsilon',\delta)$-approximates $g_i$. Putting all the $\CN_i$
   together, we obtain an invariant MPNN $\CN$ that computes a
   function $g:\CG_n\to\{0,1\}^N$. We only need to apply the linear
   transformation
   \[
     \vec x\mapsto\sum_{i=1}^N x_i\cdot y_i
   \]
   to the output of $\CN$ to obtain an approximation of $f$.
 \end{proof}

\begin{remark}
  Obviously, our construction yields MPNNs with a prohibitively large state
  space. In particular, this is true for the brute force step from
  Boolean to general functions. We doubt that there are much more
  efficient approximators, after all we make no assumption whatsoever
  on the function $f$.

  The approximation of Boolean functions
  is more interesting. It may still happen that
  the GNNs get exponentially large in $n$; this seems
  unavoidable. However, the nice thing here is that our construction
  is very adaptive and tightly linked to the descriptive complexity of
  the function we want to approximate. This deserves a more thorough
  investigation, which we leave for future work.

  As opposed to other universality results for GNNs, our construction
  needs no higher-order tensors defined on tuples of nodes, with
  practically infeasible space requirements on all but very small
  graphs. Instead, the complexity of our construction goes entirely
  into the dimension of the state space. The advantage of this is that
  we can treat this dimension as a hyperparameter that we can easily
  adapt and that gives us more fine-grained control over the space
  requirements. Our experiments show that usually in practice a small
  dimension already yields very powerful networks.
\end{remark}

\begin{remark}
  In our experiments, we found that partial RNI, which assigns random values to a fraction of all node embedding vectors, often yields very good results, sometimes better than a full RNI. There is a theoretical plausibility
  to this. For most graphs, we do not
  lose much by only initializing a small fraction of vertex embeddings, because in a few message-passing rounds GNNs can propagate the randomness and individualize the full input
  graph with our construction. On the other hand, we reduce the amount of noise our models
  have to handle when we only randomize partially.
\end{remark}

\subsection{Details of Dataset Construction}
\label{app:corePair}
There is an interesting universality result for functions defined on planar graphs. It is known that 3-WL can distinguish between planar graphs \cite{KieferPS19}. Since 4-GCNs can simulate 3-WL, this implies that functions over planar graphs can be approximated by 4-GCNs. This result can be extended to much wider graph classes, including all graph classes excluding a fixed graph as a minor \cite{GroheWL}. 

Inspired by this, we generate planar instances, and ensure that they can be distinguished by 2-WL, by carefully constraining these instances further. Hence, any GNN with 2-WL expressive power can approximate solutions to these planar instances. This, however, does not imply that these GNNs will solve \EXP in practice, but only that an appropriate approximation function exists and can theoretically be learned.

\subsubsection{Construction of \EXP}

\EXP consists of two main components, (i) a pair of \emph{cores}, which are non-isomorphic, planar, 1-WL indistinguishable,  2-WL distinguishable, and decide the satisfiability of every instance, and (ii) an additional randomly generated and satisfiable \emph{planar component}, identically added to the core pair, to add variability to \EXP and make learning more challenging. We first present both components, and then provide further details about graph encoding and planar embeddings. 

\paragraph{Core pair.}

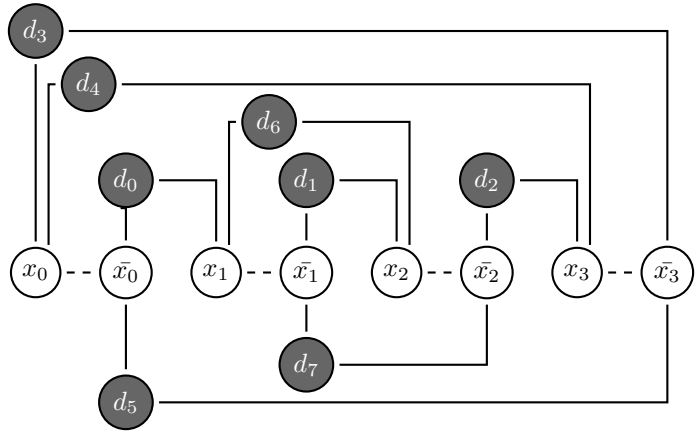
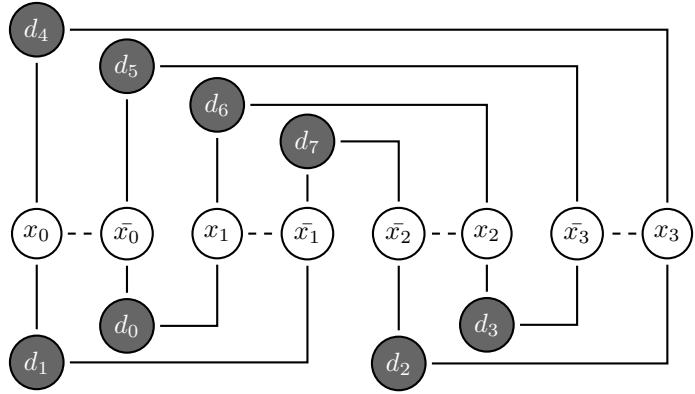
\begin{figure*}[t]  %
	\centering
	\begin{subfigure}[t]{\linewidth} 
	\centering
		\begin{tikzpicture}[node distance = 1cm,line width=0.8pt,shorten >=2pt, shorten <=2pt,-, scale=0.3]
		\tikzstyle{var} = [text width=1.2em, text centered, text=black, circle, inner sep=2pt,  draw=black, fill=white, thick]
		\tikzstyle{disj} = [text width=1.2em ,fill=gray!120,text centered, text=white, circle, inner sep=2pt,  draw=black,  thick]
		\node[var] (x0) {$x_0$};
		\node[var, right=0.5cm of x0] (x0n) {$\bar{x_0}$};
		\node[var, right = 0.5cm of x0n] (x1) {$x_1$};
		\node[var, right=0.5cm of x1] (x1n) {$\bar{x_1}$};
		\node[var, right = 0.5cm of x1n] (x2) {$x_2$};
		\node[var, right=0.5cm of x2] (x2n) {$\bar{x_2}$};
		\node[var, right = 0.5cm of x2n] (x3) {$x_3$};
		\node[var, right= 0.5cm of x3] (x3n) {$\bar{x_3}$};
		\node[disj, above= 0.5cm of  x0n] (d0) {$d_0$};
		\node[disj, above= 0.5cm of  x1n] (d1) {$d_1$};
		\node[disj, above= 0.5cm of  x2n] (d2) {$d_2$};
		\node[disj, above right = 1.5cm and 0.2 of x1] (d6) {$d_6$};
		\node[disj, above right = 2cm and 0.2cm of x0] (d4) {$d_4$};
		\node[disj, above = 2.5cm of x0] (d3) {$d_3$};
		\node[disj, below = 0.5cm  of x1n] (d7) {$d_7$};
		\node[disj, below = 1cm of x0n] (d5) {$d_5$};
		
		\draw[color=black,dashed] (x0.east) -- (x0n.west);
		\draw[color=black,dashed] (x1.east) -- (x1n.west);
		\draw[color=black,dashed] (x2.east) -- (x2n.west);
		\draw[color=black,dashed] (x3.east) -- (x3n.west);
		\draw[color=black] (x0n.north) |- (d0.south);
		\draw[color=black] (x1.north) |- (d0.east);
		\draw[color=black] (x1n.north) -- (d1.south);
		\draw[color=black] (x2.north) |- (d1.east);
		\draw[color=black] (x2n.north) -- (d2.south);
		\draw[color=black] (x3.north) |- (d2.east);
		\draw[color=black] (x0.north) -- (d3.south);
		\draw[color=black] (x3n.north) |- (d3.east);
		
		\draw[color=black] (x0.60) |- (d4.west);
		\draw[color=black] (x3.60) |- (d4.east);
		\draw[color=black] (x0n.south) -- (d5.north);
		\draw[color=black] (x3n.south) |- (d5.east);
		
		\draw[color=black] (x1.60) |- (d6.west);
		\draw[color=black] (x2.60) |- (d6.east);
		\draw[color=black] (x1n.south) --  (d7.north);
		\draw[color=black] (x2n.south) |- (d7.east);
		\end{tikzpicture}
		\caption{The encoding of the formula $\phi_1$.}\label{fig:phi1}		
	\end{subfigure}
	\\ \vspace{0.25cm}
	\begin{subfigure}[t]{\linewidth}  
	\centering
		\begin{tikzpicture}[node distance = 1cm,line width=0.8pt,shorten >=2pt, shorten <=2pt,-, scale=0.3]
		\tikzstyle{var} = [text width=1.2em, text centered, text=black, circle, inner sep=2pt,  draw=black, fill=white,  thick]
		\tikzstyle{disj} = [text width=1.2em ,fill=gray!120,text centered, text=white, circle, inner sep=2pt,  draw=black,  thick]
		\node[var] (x0) {$x_0$};
		\node[var, right=0.5cm of x0] (x0n) {$\bar{x_0}$};
		\node[var, right =0.5cm of x0n] (x1) {$x_1$};
		\node[var, right=0.5cm of x1] (x1n) {$\bar{x_1}$};
		\node[var, right =0.5cm of x1n] (x2n) {$\bar{x_2}$};
		\node[var, right=0.5cm of x2] (x2) {$x_2$};
		\node[var, right = 0.5cm of x2] (x3n) {$\bar{x_3}$};
		\node[var, right= 0.5cm of x3] (x3) {$x_3$} ;
		\node[disj, below = 0.5cm of  x0n] (d0) {$d_0$};
		\node[disj, below = 1cm of  x0] (d1) {$d_1$};
		\node[disj, below = 0.5cm of  x2] (d3) {$d_3$};
		\node[disj, below = 1cm of  x2n] (d2) {$d_2$};
		\node[disj, above = 0.5cm of x1n] (d7) {$d_7$};
		\node[disj, above = 1cm of x1] (d6) {$d_6$};
		\node[disj, above= 1.5cm of x0n] (d5) {$d_5$};
		\node[disj, above= 2cm of x0] (d4) {$d_4$};
		
		\draw[color=black,dashed] (x0.east) -- (x0n.west);
		\draw[color=black,dashed] (x1.east) -- (x1n.west);
		\draw[color=black,dashed] (x2.west) -- (x2n.east);
		\draw[color=black,dashed] (x3.west) -- (x3n.east);
		\draw[color=black] (x0.north) -- (d4.south);
		\draw[color=black] (x0.south) -- (d1.north);
		\draw[color=black] (x0n.north) -- (d5.south);		
	    \draw[color=black] (x0n.south) -- (d0.north);
	    \draw[color=black] (x1.north) -- (d6.south);		
		\draw[color=black] (x1.south) |- (d0.east);
        \draw[color=black] (x1n.north) -- (d7.south);		
		\draw[color=black] (x1n.south) |- (d1.east);	
	    \draw[color=black] (x2.north) |- (d6.east);		
		\draw[color=black] (x2.south) -- (d3.north);
        \draw[color=black] (x2n.north) |- (d7.east);		
		\draw[color=black] (x2n.south) -- (d2.north);
	    \draw[color=black] (x3.north) |- (d4.east);		
		\draw[color=black] (x3.south) |- (d2.east);
        \draw[color=black] (x3n.north) |- (d5.east);		
		\draw[color=black] (x3n.south) |- (d3.east);		
		\end{tikzpicture}
		\caption{The encoding of the formula $\phi_2$.}\label{fig:phi2}		
	\end{subfigure}
	\caption{Illustration of planar embeddings for the formulas $\phi_1$ and $\phi_2$ for $n=2$.}
\label{fig:planEmb}
\end{figure*}
In \EXP, a core pair consists of two \CNF formulas $\phi_1, \phi_2$, both defined using $2n$ variables, $n \in \Nbb^{+}$, such that $\phi_1$ is unsatisfiable and $\phi_2$ is satisfiable, and such that their graph encodings are 1-WL indistinguishable and planar. $\phi_1$ and $\phi_2$ are constructed using two structures which we refer to as \emph{variable chains} and \emph{variable bridges} respectively. 

A \emph{variable chain} $\phi_{chain}$ is defined over a set of $n \geq 2$ Boolean variables%
, and imposes that all variables be equally set. The variable chain can be defined in increasing or decreasing order over these variables. More specifically, given variables $x_i, ..., x_j$, 
\begin{align}
\text{Chain}_\text{Inc}(i, j) &= \bigwedge_{k=i}^{j-1} (\bar{x_k} \vee x_{i+(k+1)\%(j-i+1)}),~\text{and}\\
\text{Chain}_\text{Dec}(i, j) &= \bigwedge_{k=i}^{j-1} ({x_k} \vee \bar{x}_{i+(k+1)\%(j-i+1)}).
\end{align}
Additionally, a \emph{variable bridge} is defined over an even number of variables $x_0,..., x_{2n-1}$, as 
\begin{align}
{\phi_{bridge} = \bigwedge_{i=0}^{n-1} \big((x_i \vee x_{2n-1-i}) \land (\bar{x_i} \vee \bar{x}_{2n-1-i})\big)}.
\end{align}
A variable bridge makes the variables it connects forcibly have opposite values, e.g., $x_0 = \bar{x_1}$ for $n=1$. We denote a variable bridge over $x_0,..., x_{2n-1}$ as $\text{Bridge}(2n)$. 

To get $\phi_1$ and $\phi_2$, we define $\phi_1$ as a variable chain and bridge on all variables, yielding contrasting and unsatisfiable constraints. To define $\phi_2$, we ``cut'' the chain in half, such that the first $n$ variables can differ from the latter $n$, satisfying the bridge. The second half of the ``cut'' chain is then flipped to a decrementing order, which preserves the satisfiability of $\phi_2$, but maintains the planarity of the resulting graph. More specifically, this yields:
\begin{align}
\phi_1 &= \text{Chain}_\text{Inc}(0, 2n) \land \text{Bridge}(2n)\text{, and} \\
\phi_2 &= \text{Chain}_\text{Inc}(0, n) \land \text{Chain}_\text{Dec}(n, 2n) \land \text{Bridge}(2n).
\end{align}
\paragraph{Planar component.}
Following the generation of $\phi_1$ and $\phi_2$, a disjoint satisfiable planar graph component $\phi_\text{planar}$ is added. $\phi_\text{planar}$ shares no variables or disjunctions with the cores, so is primarily introduced to create noise and make learning more challenging. $\phi_\text{planar}$ is generated starting from random 2-connected (i.e., at least 2 edges must be removed to disconnect a component within the graph) bipartite planar graphs from the Plantri tool \cite{brinkmann2007fast}, such that (i)
the larger set of nodes in the graph is the variable set\footnote{Ties are broken arbitrarily if the two sets are equally sized.}, (ii)
highly-connected disjunctions are split in a planarity-preserving
fashion to maintain disjunction widths not exceeding 5, (iii) literal signs for variables are uniformly randomly assigned, and (iv) redundant disjunctions,
if any, are removed. If this $\phi_\text{planar}$ is satisfiable, then it is accepted and used. Otherwise, the formula is discarded and a new $\phi_\text{planar}$ is analogously generated until a satisfiable formula is produced.

Since the core pair and $\phi_{\text{planar}}$ are disjoint, it clearly follows that the graph encodings of $\phi_\text{planar} \land \phi_1$ and $\phi_\text{planar} \land \phi_2$ are planar and 1-WL indistinguishable. Furthermore, $\phi_\text{planar} \land \phi_1$ is satisfiable, and $\phi_\text{planar} \land \phi_2$ is not. Hence, the introduction of $\phi_\text{planar}$ maintains all the desirable core properties, all while making any generated \EXP dataset more challenging. 

The structural properties of the cores, combined with the combinatorial difficulty of \SAT, make \EXP a challenging dataset. For example, even minor formula changes, such as flipping a literal, can lead to a change in the \SAT outcome, which enables the creation of near-identical, yet semantically different instances. Moreover, \SAT is NP-complete \cite{cook1971complexity}, and remains so on planar instances \cite{hunt1998complexity}. Hence, \EXP is cast to be challenging, both from an expressiveness and computational perspective. 

\paragraph{Remark 3.} Intuitively, $\phi_1$ and $\phi_2$, generated as described, can be distinguished by 2-WL, as 2-WL can detect the break in cycles resulting from the ``cut''. In other words, 2-WL can identify that the chain has been broken in between these two formulas, and thus will return distinct colorings. Hence, $\phi_1$ and $\phi_2$ can be distinguished by 3-GCNs.
\paragraph{Graph encoding.}
We use the following graph encoding, denoted by $Enc$: (i) Every variable is encoded by two nodes, representing its positive and negative literals, and connected by an edge, (ii) Every disjunction is represented by a node, and an edge connects a literal node to a disjunction node if the literal appears in the disjunction, and (iii) Variable and disjunction nodes are encoded with their respective types. We opt for this encoding, as it is commonly used in the literature \cite{Selsam-ICLR2019}, and, for the sake of our empirical evaluation, yields planar encodings for \EXP graph pairs. 

\paragraph{Planar embeddings for core pair.}
We show planar embeddings for $Enc(\phi_1)$ and $Enc(\phi_2)$ for $n=2$ in \Cref{fig:planEmb}, and these embeddings can naturally be extended to any $n$. %
 $Enc(\phi_1)$ and $Enc(\phi_2)$ can also be shown to be 1-WL indistinguishable. This can be observed intuitively, as node neighborhoods in both graphs are identical and very regular: all variable nodes are connected to exactly one other variable node and two disjunction nodes, and all disjunction nodes are connected to exactly two variables. %

\subsubsection{Construction of \CEXP}
\label{app:CExpConstruction}
Given an \EXP dataset with $N$ pairs of graphs, we create \CEXP by selecting $N / 2$ graph pairs and modifying them to yield \Corrupt. The unmodified graph pairs are therefore exactly identical in type to \EXP instances, and we refer to these instances within \CEXP as \EXPTwo.  Then, for every graph pair, we discard the satisfiable graph and construct a new graph from a copy of the unsatisfiable graph as follows:
\begin{enumerate}
    \item Randomly introduce new literals to the existing disjunctions of the copy of the unsatisfiable graph, such that no redundancies are created (e.g., adding $x$ to a disjunction when $x$ or $\bar{x}$ is already present), until 3 literals are added \emph{and} the formula becomes satisfiable. Literal addition is done by creating new edges in the graph between disjunction and literal nodes. To do this, disjunctions with less than 5 literals are uniformly randomly selected, and the literal to add is uniformly randomly sampled from the set of all non-redundant literals for the disjunction.
    \item Once a satisfiable formula is reached, iterate sequentially over all added edges, and eliminate any edge whose removal does not restore unsatisfiability. This ensures that a minimal number of new edges, relative to the original unsatisfiable graph, are added. 
\end{enumerate}

Observe that these modifications have several interesting effects on the dataset. First, they preserve the existing UNSAT core nodes and edges, while flipping the satisfiability of their overall formulas, which makes the learning task go beyond structure identification. Second, they introduce significant new variability to the dataset, in that the planar component and cores can share edges. Finally, they make the graph pairs 1-WL distinguishable, which gives standard GNNs a chance to perform well on \Corrupt. 

\subsubsection{Dataset generation for experiments}
To create the \EXP dataset, we randomly generate 600 core pairs, where $n$ is uniformly randomly set between 2 and 4 inclusive. Then, we generate the additional planar component using Plantri, such that 500 $\phi_\text{planar}$ formulas are generated from 12-node planar bipartite planar graphs, and the remaining 100 from planar bipartite graphs with 15 nodes. 

This generation process implies that every formula has a number of variables ranging between 10 (4 core variables when $n=2$ plus a minimum 6 variables from the larger bipartite set during $\phi_\text{planar}$ generation from 12-node graphs) and 22 variables (8 core variables for $n=4$ plus a maximally-sized variable subset of 14 nodes for $\phi_\text{planar}$ generation from 15-node graphs). 

Furthermore, the number of disjunctions also ranges from 10 (8 core disjunctions for $n=2$ plus the minimum 2 disjunctions for the case where $\phi_\text{planar}$, generated from 12-node graphs, has 10 variables and 2 disjunctions) to 30 disjunctions (16 core disjunctions for $n=4$ plus at most 14 disjunctions for the case where $\phi_\text{planar}$, generated from 15-node graphs, initially has 8 variables and 7 disjunctions, which can at most lead to 14 final disjunctions following step (ii)).

\subsection{Standard Deviation of GCN-50\%RNI on \EXP over training}

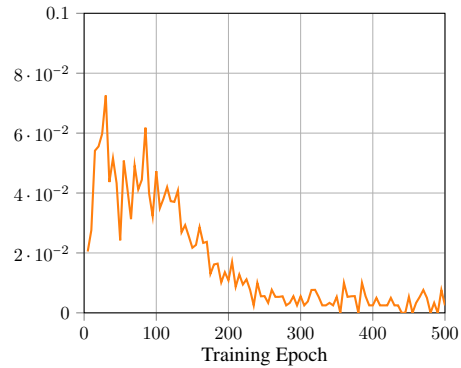
\begin{figure}
\centering
\begin{tikzpicture}[scale=0.7]

\definecolor{color0}{rgb}{0.12156862745098,0.466666666666667,0.705882352941177}
\definecolor{color4}{rgb}{1,0.498039215686275,0.0549019607843137}
\definecolor{color2}{rgb}{0.172549019607843,0.627450980392157,0.172549019607843}
\definecolor{color3}{rgb}{0.83921568627451,0.152941176470588,0.156862745098039}
\definecolor{color1}{rgb}{0.4,0.4,0.4}

\begin{axis}[
label style={font=\large},
tick label style={font=\normalsize},
tick align=outside,
tick pos=left,
x grid style={white!69.01960784313725!black},
xlabel={Training Epoch},
xmajorgrids,
xmin=0, xmax=500, 
y grid style={white!69.01960784313725!black},
ymajorgrids,
ymin=0, ymax=0.10
]

\addplot [line width=0.4mm, color4]
table [row sep=\\]{%
5 0.020480342879074177\\
10 0.027588242262078077\\
15 0.05413460589637239\\
20 0.05555277770832988\\
25 0.059773879468097674\\
30 0.07255744850346749\\
35 0.043652669512362643\\
40 0.05161287522400674\\
45 0.04330127018922194\\
50 0.024209731743889924\\
55 0.050826502273256365\\
60 0.04133198922545748\\
65 0.03128054063187242\\
70 0.04916666666666667\\
75 0.04119735698102758\\
80 0.04444097208657795\\
85 0.06182412330330469\\
90 0.03993049516903648\\
95 0.03227486121839516\\
100 0.047316898555261305\\
105 0.03497022543059546\\
110 0.038151743807531995\\
115 0.041965594373380564\\
120 0.037277115410575073\\
125 0.037043517951488115\\
130 0.040858835573770895\\
135 0.02688710967483614\\
140 0.02929732638541157\\
145 0.025603819159562044\\
150 0.021730674684008827\\
155 0.022684429314693673\\
160 0.028625940062196122\\
165 0.02334820954353649\\
170 0.02370243775554648\\
175 0.013017082793177762\\
180 0.01615893285805442\\
185 0.01643590243609669\\
190 0.010172129679778086\\
195 0.013565683830083088\\
200 0.010833333333333339\\
205 0.01699673171197596\\
210 0.00870025542409214\\
215 0.012801909579781013\\
220 0.009464847243000464\\
225 0.01118033988749895\\
230 0.007637626158259735\\
235 0.002499999999999991\\
240 0.010034662148993576\\
245 0.005527707983925673\\
250 0.0055901699437494795\\
255 0.0033333333333333214\\
260 0.0076829537144107425\\
265 0.005335936864527384\\
270 0.005335936864527384\\
275 0.005527707983925671\\
280 0.002499999999999991\\
285 0.0033333333333333214\\
290 0.005527707983925673\\
295 0.002499999999999991\\
300 0.005527707983925673\\
305 0.002499999999999991\\
310 0.0038188130791298527\\
315 0.007637626158259738\\
320 0.0076829537144107425\\
325 0.005335936864527385\\
330 0.002499999999999991\\
335 0.002499999999999991\\
340 0.0033333333333333214\\
345 0.002499999999999991\\
350 0.005335936864527385\\
355 0.0\\
360 0.010034662148993576\\
365 0.005335936864527385\\
370 0.005527707983925673\\
375 0.0055901699437494795\\
380 0.0\\
385 0.010034662148993576\\
390 0.005527707983925671\\
395 0.002499999999999991\\
400 0.002499999999999991\\
405 0.005000000000000016\\
410 0.002499999999999991\\
415 0.002499999999999991\\
420 0.002499999999999991\\
425 0.005000000000000016\\
430 0.002499999999999991\\
435 0.002499999999999991\\
440 0.0\\
445 0.0\\
450 0.005335936864527384\\
455 0.0\\
460 0.0033333333333333214\\
465 0.005335936864527384\\
470 0.007637626158259738\\
475 0.005000000000000016\\
480 0.0\\
485 0.0033333333333333214\\
490 0.0\\
495 0.007637626158259738\\
500 0.002499999999999991\\
};

\end{axis}
\end{tikzpicture}
\caption{Standard deviation of test accuracy over all 10 validation splits for GCN-50\%RNI on \EXP.}\label{fig:stdev}
\end{figure}
\begin{figure*}[t!]
	\centering
	\begin{subfigure}{.48\textwidth}
		\centering
		\begin{tikzpicture}[scale=0.67]

\definecolor{ThreeGCN}{rgb}{0,0.516,0.66}
\definecolor{GCNStandard}{rgb}{0.98,0.79,0.26}
\definecolor{GCNRNI12}{rgb}{0,0.16,0.39}
\definecolor{GCNRNI50}{rgb}{0.40,0.66,0.26}
\definecolor{GCNRNI87}{rgb}{0.84,0.22,0.05}
\definecolor{GCNRNIFull}{rgb}{0.63,0.39,0.7}

\begin{axis}[
legend cell align={left},
legend columns=3,
legend entries={{\ThreeGNNFull},{\GCNRNI}, {GCN-50\%}, {GCN-87.5\%}, {GCN-12.5\%}},
legend style={at={(-0.07, -0.21)}, font=\normalsize, line width=0.4mm, anchor=north west, draw=white!80.0!black},
label style={font=\large},
tick label style={font=\normalsize},
tick align=outside,
tick pos=left,
x grid style={white!69.01960784313725!black},
xlabel={Epoch},
xmajorgrids,
xmin=0, xmax=500, 
y grid style={white!69.01960784313725!black},
ymajorgrids,
ymin=50, ymax=100
]

\addplot [line width=0.4mm, ThreeGCN]
table [row sep=\\]{%
1 52\\
2 62.41\\
3 81.1\\
4 93.1\\
5 96.4\\
10 99.3\\
15 99.33\\
20 99.58\\
25 99.75\\
30 99.75\\
35 99.75\\
40 99.58\\
45 99.83\\
90 99.83\\
95 99.83\\
100 99.83\\
};

\addplot [line width=0.4mm, GCNRNIFull]
table [row sep=\\]{%
0 50\\
10 51.7\\
20 56.2\\
30 64.6\\
40 71.85\\
50 74.21\\
60 77.18\\
70 80.45\\
80 83.06\\
90 85.98\\ 
100 88.48\\
110 89.21\\
120 90.33\\
130 90.8\\
140 91.56\\
150 92.19\\
160 93.31\\
170 93.79\\
180 93.71\\
190 94.27\\
200 94.05\\
210 94.86\\
220 94.59\\
230 95\\
240 95.12\\
250 95.72\\
260 95.10\\
270 95.70\\
280 95.87\\
290 96.27\\
300 95.73\\
310 96.36\\
320 96.17\\
330 96.53\\
340 96.73\\
350 96.59\\
360 96.40\\
370 97.01\\
380 97.11\\
390 97.04\\
400 97.28\\
410 97.36\\
420 97.13\\
430 97.66\\
440 97.56\\
450 97.37\\
460 97.54\\
470 97.48\\
480 97.55\\
490 97.46\\
500 98.01\\
};

\addplot [line width=0.4mm, GCNRNI50]
table [row sep=\\]{%
0 50\\
10 52.17\\
20 57.33\\
30 61.58\\
40 69.75\\
50 79.67\\
60 83.08\\
70 85.58\\
80 89.00\\
90 89.25\\ 
100 90.92\\
110 91.58\\
120 92.75\\
130 92.33\\
140 94.33\\
150 93.83\\
160 95.17\\
170 95.42\\
180 96.42\\
190 96.50\\
200 95.67\\
210 97.75\\
220 98.42\\
230 97.83\\
240 98.00\\
250 98.42\\
260 98.33\\
270 98.58\\
280 99.00\\
290 99.42\\
300 99.33\\
310 99.67\\
320 99.58\\
330 99.50\\
340 99.58\\
350 99.67\\
360 99.58\\
370 99.75\\
380 99.58\\
390 99.67\\
400 99.17\\
410 99.83\\
420 99.83\\
430 99.42\\
440 99.50\\
450 99.83\\
460 99.92\\
470 99.75\\
480 99.83\\
490 99.75\\
500 99.92\\
};

\addplot [line width=0.4mm, GCNRNI87] %
table [row sep=\\]{%
0 50\\
10 53.67\\
20 52.41\\
30 61.33\\
40 68.83\\
50 69.25\\
60 76.25\\
70 79.83\\
80 84.83\\
90 85.92\\ 
100 87.33\\
110 91.58\\
120 91.83\\
130 93.00\\
140 93.92\\
150 95.42\\
160 96.25\\
170 94.75\\
180 96.33\\
190 97.67\\
200 97.17\\
210 96.75\\
220 97.50\\
230 97.83\\
240 97.50\\
250 98.17\\
260 97.42\\
270 98.08\\
280 98.33\\
290 97.83\\
300 99.00\\
310 99.08\\
320 98.67\\
330 98.75\\
340 98.58\\
350 98.92\\
360 99.58\\
370 99.25 \\
380 99.33 \\
390 99.00 \\
400 99.58 \\
410 99.75 \\
420 99.67 \\
430 99.67 \\
440 99.58 \\
450 99.75 \\
460 99.33 \\
470 99.67 \\
480 99.58 \\
490 99.75 \\
500 99.83 \\
};

\addplot [line width=0.4mm, GCNRNI12] %
table [row sep=\\]{%
0 50\\
10 51.50\\
20 54.67\\
30 65.83\\
40 76.83\\
50 80.75\\
60 82.25\\
70 85.00\\
80 85.75\\
90 88.58\\ 
100 92.33\\
110 94.17\\
120 95.17\\
130 97.33\\
140 98.33\\
150 98.75\\
160 99.75\\
170 99.67\\
180 99.58\\
190 99.58\\
200 99.67\\
210 99.83\\
220 99.75\\
230 99.83\\
240 99.92\\
250 99.75\\
260 99.92\\
270 99.92\\
280 100.0\\
290 99.92\\
300 100.0\\
310 99.83\\
320 99.92\\
330 100.0\\
340 99.75\\
350 99.83\\
360 99.83\\
370 99.92\\
380 100.0 \\
390 99.50 \\
400 99.75 \\
410 99.92 \\
420 99.75 \\
430 99.50 \\
440 99.92 \\
450 100.0 \\
460 99.93 \\
470 99.75 \\
480 100.0 \\
490 99.92 \\
500 99.92 \\
};
\end{axis}
\end{tikzpicture}
		\caption{Learning curves on \EXP.}
		\label{app:fig:modelConvergence}
	\end{subfigure}%
	\begin{subfigure}{.48\textwidth}
		\centering
		\begin{tikzpicture}[scale=0.67]

\definecolor{ThreeGCN}{rgb}{0,0.516,0.66}
\definecolor{GCNStandard}{rgb}{0.98,0.79,0.26}
\definecolor{GCNRNI12}{rgb}{0,0.16,0.39}
\definecolor{GCNRNI50}{rgb}{0.40,0.66,0.26}
\definecolor{GCNRNI87}{rgb}{0.84,0.22,0.05}
\definecolor{GCNRNIFull}{rgb}{0.63,0.39,0.7}

\begin{axis}[
legend cell align={left},
legend columns=3,
legend entries={{\ThreeGNNFull},{\GCNRNI}, {GCN-50\%}, {GCN-87.5\%}, {GCN-12.5\%}},
legend style={at={(-0.07, -0.21)}, font=\normalsize, line width=0.4mm, anchor=north west, draw=white!80.0!black},
label style={font=\large},
tick label style={font=\normalsize},
tick align=outside,
tick pos=left,
x grid style={white!69.01960784313725!black},
xlabel={Epoch},
xmajorgrids,
xmin=0, xmax=1000, 
y grid style={white!69.01960784313725!black},
ymajorgrids,
ymin=50, ymax=100
]

\addplot [line width=0.4mm, ThreeGCN]
table [row sep=\\]{%
10 50.33333333333333\\
20 68.66666666666667\\
30 87.66666666666667\\
40 87.66666666666667\\
50 97.66666666666669\\
60 96.66666666666666\\
70 99.0\\
80 99.0\\
90 99.0\\
100 99.33333333333334\\
110 99.33333333333334\\
120 99.33333333333334\\
130 99.0\\
140 99.33333333333334\\
150 99.0\\
160 98.66666666666667\\
170 99.33333333333334\\
180 99.33333333333334\\
190 97.33333333333333\\
200 99.33333333333334\\
};

\addplot [line width=0.4mm, GCNRNIFull]
table [row sep=\\]{%
10 50.0\\
20 50.33333333333333\\
30 50.66666666666666\\
40 51.66666666666667\\
50 49.66666666666667\\
60 52.666666666666664\\
70 46.666666666666664\\
80 56.33333333333332\\
90 56.666666666666664\\
100 59.00000000000001\\
110 59.0\\
120 66.66666666666667\\
130 63.66666666666667\\
140 64.33333333333331\\
150 64.66666666666666\\
160 64.33333333333333\\
170 70.33333333333334\\
180 67.66666666666666\\
190 68.0\\
200 70.33333333333334\\
210 69.66666666666666\\
220 68.0\\
230 71.66666666666667\\
240 71.0\\
250 76.66666666666666\\
260 74.99999999999999\\
270 76.0\\
280 76.0\\
290 78.0\\
300 78.33333333333333\\
310 79.33333333333333\\
320 80.66666666666666\\
330 78.66666666666667\\
340 83.0\\
350 82.33333333333333\\
360 82.0\\
370 85.33333333333333\\
380 83.33333333333333\\
390 81.66666666666667\\
400 83.0\\
410 88.33333333333334\\
420 84.66666666666667\\
430 87.00000000000001\\
440 89.00000000000003\\
450 85.66666666666669\\
460 90.0\\
470 90.00000000000001\\
480 88.00000000000001\\
490 88.66666666666667\\
500 90.66666666666669\\
510 86.99999999999999\\
520 90.0\\
530 91.33333333333333\\
540 92.33333333333334\\
550 90.33333333333333\\
560 89.33333333333333\\
570 92.66666666666667\\
580 95.0\\
590 91.33333333333334\\
600 91.00000000000001\\
610 92.33333333333334\\
620 92.33333333333334\\
630 92.66666666666667\\
640 94.66666666666667\\
650 93.0\\
660 92.66666666666667\\
670 93.0\\
680 94.0\\
690 91.66666666666667\\
700 92.33333333333334\\
710 92.00000000000001\\
720 95.0\\
730 91.66666666666667\\
740 94.0\\
750 93.33333333333333\\
760 96.33333333333333\\
770 93.0\\
780 95.33333333333334\\
790 92.66666666666667\\
800 92.33333333333334\\
810 92.66666666666667\\
820 93.0\\
830 96.0\\
840 94.66666666666667\\
850 96.66666666666669\\
860 95.0\\
870 96.0\\
880 97.0\\
890 95.33333333333334\\
900 95.0\\
910 94.33333333333334\\
920 97.0\\
930 94.0\\
940 95.66666666666667\\
950 97.66666666666669\\
960 95.66666666666667\\
970 94.66666666666667\\
980 95.0\\
990 95.66666666666667\\
1000 95.66666666666667\\
};

\addplot [line width=0.4mm, GCNRNI50]
table [row sep=\\]{%
10 52.666666666666664\\
20 49.66666666666666\\
30 51.0\\
40 50.33333333333333\\
50 54.666666666666664\\
60 52.666666666666664\\
70 49.333333333333336\\
80 52.666666666666664\\
90 52.666666666666664\\
100 55.666666666666664\\
110 55.99999999999999\\
120 56.666666666666664\\
130 52.0\\
140 58.66666666666666\\
150 61.66666666666667\\
160 65.99999999999999\\
170 61.0\\
180 67.00000000000001\\
190 65.99999999999999\\
200 68.0\\
210 68.66666666666667\\
220 68.33333333333333\\
230 68.66666666666667\\
240 69.0\\
250 79.66666666666666\\
260 73.0\\
270 78.99999999999999\\
280 74.0\\
290 75.33333333333333\\
300 80.0\\
310 76.66666666666669\\
320 80.66666666666666\\
330 86.66666666666669\\
340 84.66666666666669\\
350 84.66666666666667\\
360 88.33333333333334\\
370 92.33333333333334\\
380 90.66666666666666\\
390 90.00000000000001\\
400 92.00000000000001\\
410 91.33333333333333\\
420 93.0\\
430 95.66666666666667\\
440 92.33333333333333\\
450 94.0\\
460 94.33333333333334\\
470 96.66666666666666\\
480 96.66666666666666\\
490 97.66666666666669\\
500 98.33333333333334\\
510 97.00000000000001\\
520 96.66666666666666\\
530 99.66666666666667\\
540 98.33333333333334\\
550 98.33333333333334\\
560 98.33333333333334\\
570 98.33333333333334\\
580 99.0\\
590 99.33333333333334\\
600 99.66666666666667\\
610 98.66666666666667\\
620 98.66666666666667\\
630 99.0\\
640 99.0\\
650 99.0\\
660 99.33333333333334\\
670 99.0\\
680 100.0\\
690 99.33333333333334\\
700 99.66666666666667\\
710 100.0\\
720 99.66666666666667\\
730 100.0\\
740 99.66666666666667\\
750 100.0\\
760 100.0\\
770 100.0\\
780 100.0\\
790 99.66666666666667\\
800 99.66666666666667\\
810 100.0\\
820 100.0\\
830 100.0\\
840 100.0\\
850 99.66666666666667\\
860 100.0\\
870 100.0\\
880 100.0\\
890 100.0\\
900 99.66666666666667\\
910 99.66666666666667\\
920 100.0\\
930 100.0\\
940 100.0\\
950 100.0\\
960 100.0\\
970 100.0\\
980 100.0\\
990 100.0\\
1000 99.33333333333334\\
};

\addplot [line width=0.4mm, GCNRNI87] %
table [row sep=\\]{%
10 51.0\\
20 52.666666666666664\\
30 48.66666666666667\\
40 50.66666666666666\\
50 53.99999999999999\\
60 51.66666666666666\\
70 50.66666666666666\\
80 55.333333333333336\\
90 55.00000000000001\\
100 49.66666666666666\\
110 55.00000000000001\\
120 57.33333333333333\\
130 55.66666666666665\\
140 56.666666666666664\\
150 60.0\\
160 64.99999999999999\\
170 68.0\\
180 75.66666666666667\\
190 72.66666666666666\\
200 81.0\\
210 77.66666666666666\\
220 86.66666666666666\\
230 90.0\\
240 84.66666666666667\\
250 89.66666666666667\\
260 91.33333333333333\\
270 91.33333333333334\\
280 93.0\\
290 95.0\\
300 94.0\\
310 97.66666666666669\\
320 98.00000000000001\\
330 98.66666666666667\\
340 98.00000000000001\\
350 99.33333333333334\\
360 99.33333333333334\\
370 99.66666666666667\\
380 99.66666666666667\\
390 99.66666666666667\\
400 99.66666666666667\\
410 99.66666666666667\\
420 99.66666666666667\\
430 99.66666666666667\\
440 100.0\\
450 100.0\\
460 99.66666666666667\\
470 99.0\\
480 100.0\\
490 99.66666666666667\\
500 100.0\\
510 99.33333333333334\\
520 99.66666666666667\\
530 100.0\\
540 100.0\\
550 100.0\\
560 99.66666666666667\\
570 99.66666666666667\\
580 100.0\\
590 100.0\\
600 100.0\\
610 100.0\\
620 99.0\\
630 100.0\\
640 99.66666666666667\\
650 99.66666666666667\\
660 99.33333333333334\\
670 100.0\\
680 100.0\\
690 100.0\\
700 100.0\\
710 99.66666666666667\\
720 100.0\\
730 98.00000000000001\\
740 99.33333333333334\\
750 100.0\\
760 100.0\\
770 99.66666666666667\\
780 100.0\\
790 99.66666666666667\\
800 99.66666666666667\\
810 100.0\\
820 100.0\\
830 99.66666666666667\\
840 99.33333333333334\\
850 100.0\\
860 100.0\\
870 100.0\\
880 99.66666666666667\\
890 100.0\\
900 99.66666666666667\\
910 100.0\\
920 98.66666666666667\\
930 100.0\\
940 100.0\\
950 100.0\\
960 99.33333333333334\\
970 99.66666666666667\\
980 100.0\\
990 100.0\\
1000 100.0\\
};

\addplot [line width=0.4mm, GCNRNI12] %
table [row sep=\\]{%
10 50.0\\
20 51.0\\
30 50.0\\
40 54.666666666666664\\
50 52.666666666666664\\
60 55.00000000000001\\
70 54.666666666666664\\
80 50.0\\
90 51.33333333333333\\
100 55.666666666666664\\
110 53.33333333333334\\
120 57.666666666666664\\
130 56.666666666666664\\
140 55.666666666666664\\
150 57.99999999999999\\
160 60.66666666666667\\
170 60.33333333333333\\
180 69.0\\
190 68.0\\
200 70.0\\
210 74.33333333333334\\
220 70.66666666666667\\
230 72.33333333333333\\
240 78.0\\
250 78.0\\
260 78.0\\
270 76.33333333333334\\
280 84.33333333333334\\
290 83.33333333333334\\
300 83.33333333333334\\
310 87.33333333333334\\
320 89.33333333333333\\
330 92.33333333333333\\
340 94.33333333333334\\
350 95.33333333333334\\
360 97.0\\
370 97.33333333333334\\
380 98.66666666666667\\
390 99.0\\
400 98.33333333333334\\
410 99.0\\
420 98.33333333333334\\
430 99.33333333333334\\
440 99.33333333333334\\
450 99.0\\
460 99.33333333333334\\
470 99.66666666666667\\
480 99.33333333333334\\
490 100.0\\
500 99.33333333333334\\
510 98.33333333333334\\
520 100.0\\
530 99.66666666666667\\
540 100.0\\
550 100.0\\
560 100.0\\
570 99.66666666666667\\
580 99.0\\
590 99.33333333333334\\
600 99.66666666666667\\
610 99.0\\
620 99.66666666666667\\
630 99.66666666666667\\
640 99.33333333333334\\
650 99.66666666666667\\
660 99.66666666666667\\
670 99.66666666666667\\
680 100.0\\
690 99.33333333333334\\
700 99.33333333333334\\
710 100.0\\
720 98.66666666666667\\
730 100.0\\
740 100.0\\
750 99.66666666666667\\
760 100.0\\
770 99.66666666666667\\
780 100.0\\
790 100.0\\
800 100.0\\
810 99.66666666666667\\
820 100.0\\
830 100.0\\
840 99.66666666666667\\
850 99.33333333333334\\
860 100.0\\
870 99.66666666666667\\
880 99.66666666666667\\
890 99.33333333333334\\
900 99.66666666666667\\
910 100.0\\
920 99.66666666666667\\
930 100.0\\
940 100.0\\
950 100.0\\
960 99.66666666666667\\
970 100.0\\
980 100.0\\
990 100.0\\
1000 100.0\\
};
\end{axis}
\end{tikzpicture}
		\caption{Learning curves on \SparseEXP.}
		\label{app:fig:modelConvergenceSparse}
	\end{subfigure}%
	\caption{Model convergence results for Experiment 1 on the datasets \EXP and \SparseEXP.}
	\label{app:fig:resEXP}
\end{figure*}
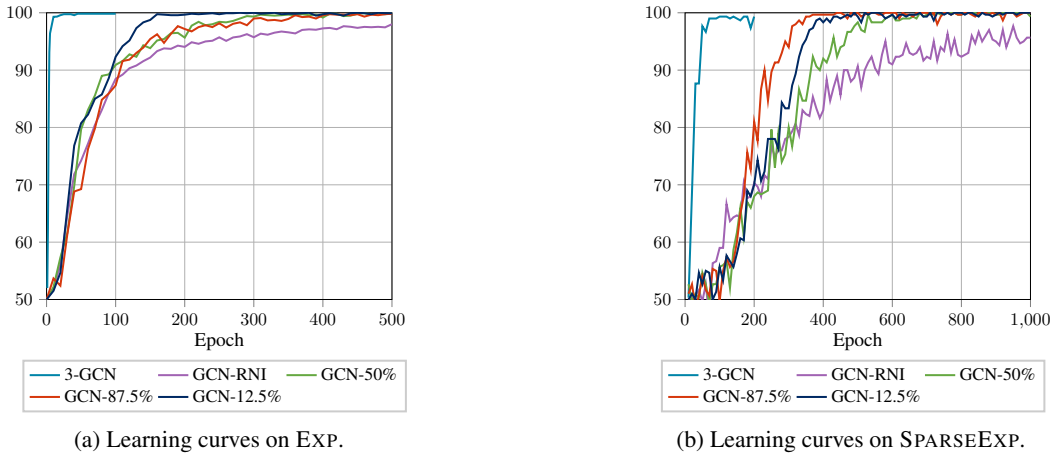
 In this subsection, we investigate the variability of \GCNRNI~learning across validation folds, and do so with a representative model and dataset, namely the semi-randomized GCN-50\%RNI model and the standard \EXP dataset. The standard deviation of the test accuracy of GCN-50\%RNI over \EXP, across all 10 cross-validation folds relative to the number of epochs, is shown in \Cref{fig:stdev}. From this figure, we see that standard deviation spikes sharply at the start of training, and only begins dropping after 100 epochs. This suggests that the learning behavior of GCN-50\% RNI is quite variable, sometimes requiring few epochs to converge, and in other cases requiring a very high number of epochs. Furthermore, standard deviation converges to almost zero following 200 epochs, corresponding to the phase where all validation folds have achieved near-perfect test performance. From these findings, we further confirm that RNI introduces volatility to GCN training, this time manifesting in variable convergence times across validation folds, but that this volatility does not ultimately hinder convergence and performance, as all folds eventually reach satisfactory performance within a reasonable amount of epochs, and subsequently stabilize. 

\subsection{Additional Experiments}
In addition to the experiments in the main body of the paper, we additionally evaluate RNI on sparser analog datasets to \EXP and \CEXP, namely \SparseEXP and \SparseCEXP. These datasets only contain 25\% of the number of instances of their original counterparts, and are used to study the behavior and impact of RNI when data is sparse.

\subsubsection{Experiment 1: \SparseEXP}
In this experiment, we generate \SparseEXP analogously to \EXP, except that this dataset only consists of 150 graph pairs, i.e., 300 graphs in total. We then train \ThreeGNNFull for 200 epochs, and all other systems for 1000 epochs on \SparseEXP, as opposed to 100 and 500 respectively for \EXP, to give all evaluated models a better opportunity to compensate for the smaller dataset size. We show the learning curves for all models on \SparseEXP, and reproduce the original figure for \EXP, in \Cref{app:fig:resEXP} for easier comparison.

First, we observe that all models converge slower on \SparseEXP compared to \EXP. This is not surprising, as a lower data availability makes learning a well-performing function slower and more challenging. More specifically, sparsity implies that (i) fewer weight updates are made per epoch, and (ii) these updates are of lower quality, as they are computed from a less complete dataset. Nonetheless, the same convergence patterns for \GCNRNI~models and \ThreeGNNFull are also visible in this setting, further highighting the increased convergence time required by \GCNRNI~models. 

We also observe that all \GCNRNI~models, though also eventually converging, do so in a more volatile fashion. Indeed, \GCNRNI~models suffer from the sparseness of the dataset, as this makes them more sensitive to RNI. As a result, these models require more training to effectively learn robustness against RNI values, and learn this from a smaller sample set, increasing their variability further. Moreover, the nature of \SparseEXP makes learning more difficult, as it fully relies on RNI for MPNNs to have a chance of achieving above-random performance, and thus encourages MPNNs to fit specific RNI values. Hence, RNI introduces significant volatility and variability to training, particularly with sparser data, and requires substantial training and epochs for \GCNRNI~models to effectively develop a robustness to RNI instantiations. 

\subsubsection{Experiment 2: \SparseCEXP}
\begin{figure*}[t!]
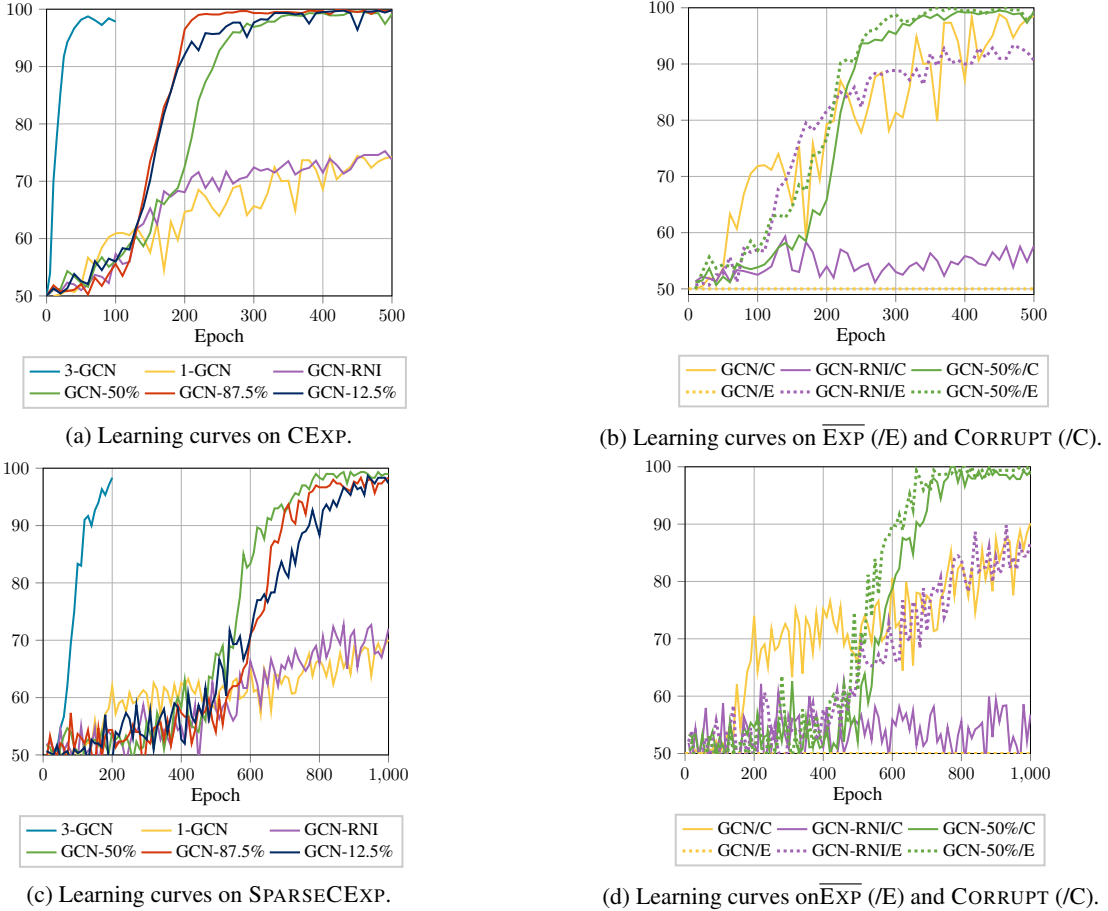

	\centering
	\begin{subfigure}{.48\textwidth}
		\centering
		\begin{tikzpicture}[scale=0.67]

\definecolor{ThreeGCN}{rgb}{0,0.516,0.66}
\definecolor{GCNStandard}{rgb}{0.98,0.79,0.26}
\definecolor{GCNRNI12}{rgb}{0,0.16,0.39}
\definecolor{GCNRNI50}{rgb}{0.40,0.66,0.26}
\definecolor{GCNRNI87}{rgb}{0.84,0.22,0.05}
\definecolor{GCNRNIFull}{rgb}{0.63,0.39,0.7}

\begin{axis}[
legend cell align={left},
legend columns=3,
legend entries={{\ThreeGNNFull},{\OneGCN},{\GCNRNI}, {GCN-50\%}, {GCN-87.5\%}, {GCN-12.5\%}},
legend style={at={(-0.07, -0.21)}, font=\normalsize, line width=0.4mm, anchor=north west, draw=white!80.0!black},
label style={font=\large},
tick label style={font=\normalsize},
tick align=outside,
tick pos=left,
x grid style={white!69.01960784313725!black},
xlabel={Epoch},
xmajorgrids,
xmin=0, xmax=500, 
y grid style={white!69.01960784313725!black},
ymajorgrids,
ymin=50, ymax=100
]

\addplot [line width=0.4mm, ThreeGCN]
table [row sep=\\]{%
0 50\\
5 53.91 \\
10 70.08\\
15 78.00 \\
20 85.33 \\
25 91.74 \\
30 94.25 \\
40 96.67 \\
50 98.17 \\
60 98.75 \\
70 98.08 \\
80 97.25 \\
90 98.42 \\ 
100 97.83\\
};

\addplot [line width=0.4mm,GCNStandard]
table [row sep=\\]{
10 50.0\\
20 50.25000000000001\\
30 51.0\\
40 50.66666666666666\\
50 51.83333333333333\\
60 56.66666666666668\\
70 55.333333333333336\\
80 58.41666666666667\\
90 60.25\\
100 60.916666666666664\\
110 61.0\\
120 60.583333333333336\\
130 61.999999999999986\\
140 60.083333333333336\\
150 57.58333333333333\\
160 62.66666666666667\\
170 54.49999999999999\\
180 62.916666666666664\\
190 59.75\\
200 64.66666666666667\\
210 64.91666666666667\\
220 68.5\\
230 67.33333333333333\\
240 65.33333333333333\\
250 63.91666666666668\\
260 66.08333333333334\\
270 68.83333333333333\\
280 69.25000000000001\\
290 64.08333333333334\\
300 65.66666666666667\\
310 65.25000000000001\\
320 67.91666666666666\\
330 72.41666666666666\\
340 70.0\\
350 70.08333333333334\\
360 64.91666666666667\\
370 73.66666666666666\\
380 73.66666666666667\\
390 72.0\\
400 68.58333333333333\\
410 74.16666666666666\\
420 70.33333333333334\\
430 71.58333333333333\\
440 72.5\\
450 74.41666666666669\\
460 74.0\\
470 72.33333333333334\\
480 73.41666666666666\\
490 74.0\\
500 74.25\\};

\addplot [line width=0.4mm, GCNRNIFull] %
table [row sep=\\]{%
0 50\\
10 51.08\\
20 50.67\\
30 52.33\\
40 52\\
50 51\\
60 52\\
70 53.67\\
80 53.33\\
90 52.25\\ 
100 57.33\\
110 55.58\\
120 55.92\\
130 61.67\\
140 62.58\\
150 65.25\\
160 62.5\\
170 68.25\\
180 67.33\\
190 68.33\\
200 68.08\\
210 70.67\\
220 71.58\\
230 68.83\\
240 70.58\\
250 68.33\\
260 71.67\\
270 69.58\\
280 70.42\\
290 70.75\\
300 72.42\\
310 71.83\\
320 72.17\\
330 71.50\\
340 72.50\\
350 73.50\\
360 71.17\\
370 72.00 \\
380 72.33 \\
390 73.58 \\
400 71.50 \\
410 73.92\\
420 72.83\\
430 71.33\\
440 72.08\\
450 74.00\\
460 74.58\\
470 74.58\\
480 74.58\\
490 75.25\\
500 73.75\\
};

\addplot [line width=0.4mm, GCNRNI50] %
table [row sep=\\]{%
0 50\\
10 51.17\\
20 51.17\\
30 54.33\\
40 53.42\\
50 52.25\\
60 51.58\\
70 55.08\\
80 56.75\\
90 55.08\\ 
100 56.08\\
110 57.25\\
120 59.00\\
130 60.33\\
140 58.67\\
150 61.00\\
160 66.75\\
170 66.00\\
180 67.67\\
190 68.83\\
200 72.58\\
210 77.58\\
220 84.08\\
230 87.33\\
240 89.58\\
250 92.75\\
260 94.41\\
270 96.00\\
280 95.92\\
290 97.50\\
300 96.92\\
310 97.17\\
320 97.83\\
330 97.83\\
340 98.75\\
350 99.00\\
360 98.92\\
370 98.83 \\
380 99.33 \\
390 99.25 \\
400 99.33 \\
410 98.92 \\
420 98.92 \\
430 99.00 \\
440 99.58 \\
450 99.58 \\
460 99.50 \\
470 99.25 \\
480 99.33 \\
490 97.42 \\
500 99.17 \\
};

\addplot [line width=0.4mm, GCNRNI87] %
table [row sep=\\]{%
0 50\\
10 51.83\\
20 50.83\\
30 50.83\\
40 51.08\\
50 52.08\\
60 50.25\\
70 53.17\\
80 51.75\\
90 53.83\\ 
100 55.58\\
110 53.50\\
120 56.17\\
130 61.42\\
140 67.00\\
150 73.50\\
160 77.58\\
170 83.00\\
180 85.67\\
190 90.42\\
200 96.50\\
210 98.08\\
220 99.00\\
230 99.17\\
240 99.08\\
250 99.08\\
260 99.42\\
270 99.42\\
280 99.67\\
290 99.67\\
300 99.33\\
310 99.33\\
320 99.25\\
330 99.50\\
340 99.50\\
350 99.50\\
360 99.42\\
370 99.33 \\
380 99.75 \\
390 99.67 \\
400 99.67 \\
410 99.08 \\
420 99.92 \\
430 99.83 \\
440 99.67 \\
450 99.50 \\
460 99.75 \\
470 99.17 \\
480 99.83 \\
490 99.67 \\
500 99.75 \\
};

\addplot [line width=0.4mm, GCNRNI12] %
table [row sep=\\]{%
0 50\\
10 51.33\\
20 50.42\\
30 51.33\\
40 53.83\\
50 52.58\\
60 52.08\\
70 56.08\\
80 54.50\\
90 56.50\\ 
100 56.00\\
110 58.33\\
120 58.08\\
130 62.17\\
140 65.41\\
150 70.25\\
160 76.67\\
170 81.75\\
180 85.73\\
190 89.67\\
200 92.08\\
210 94.33\\
220 92.83\\
230 95.83\\
240 95.67\\
250 95.75\\
260 97.00\\
270 97.67\\
280 97.67\\
290 95.17\\
300 97.67\\
310 98.25\\
320 98.17\\
330 99.33\\
340 99.33\\
350 99.33\\
360 99.08\\
370 99.25 \\
380 97.50 \\
390 99.33 \\
400 99.50 \\
410 99.58 \\
420 99.67 \\
430 99.75 \\
440 99.33 \\
450 96.42 \\
460 99.58 \\
470 99.58 \\
480 99.75 \\
490 99.42 \\
500 99.83 \\
};
\end{axis}
\end{tikzpicture}
		\caption{Learning curves on \CEXP.}
		\label{app:fig:modelConvergencePlus}
	\end{subfigure}%
	\begin{subfigure}{.48\textwidth}
		\centering
		\begin{tikzpicture}[scale=0.67]

\definecolor{ThreeGCN}{rgb}{0,0.516,0.66}
\definecolor{GCNStandard}{rgb}{0.98,0.79,0.26}
\definecolor{GCNRNI12}{rgb}{0,0.16,0.39}
\definecolor{GCNRNI50}{rgb}{0.40,0.66,0.26}
\definecolor{GCNRNI87}{rgb}{0.84,0.22,0.05}
\definecolor{GCNRNIFull}{rgb}{0.63,0.39,0.7}

\begin{axis}[
legend cell align={left},
legend columns=3,
legend entries={{GCN/C},{\GCNRNI/C}, {GCN-50\%/C},{GCN/E},{\GCNRNI/E}, {GCN-50\%/E}},
legend style={at={(-0.03, -0.21)}, font=\normalsize, line width=0.4mm, anchor=north west, draw=white!80.0!black},
label style={font=\large},
tick label style={font=\normalsize},
tick align=outside,
tick pos=left,
x grid style={white!69.01960784313725!black},
xlabel={Epoch},
xmajorgrids,
xmin=0, xmax=500, 
y grid style={white!69.01960784313725!black},
ymajorgrids,
ymin=49.0, ymax=100
]

\addplot [line width=0.4mm, GCNStandard, mark size=1pt]
table [row sep=\\]{%
10 50.0\\
20 50.5\\
30 52.0\\
40 51.33333333333334\\
50 53.666666666666664\\
60 63.33333333333332\\
70 60.66666666666667\\
80 66.83333333333333\\
90 70.5\\
100 71.83333333333333\\
110 72.0\\
120 71.16666666666667\\
130 74.00000000000001\\
140 70.16666666666667\\
150 65.16666666666666\\
160 75.33333333333333\\
170 59.0\\
180 75.83333333333333\\
190 69.5\\
200 79.33333333333333\\
210 79.83333333333333\\
220 87.00000000000001\\
230 84.66666666666667\\
240 80.66666666666666\\
250 77.83333333333333\\
260 82.16666666666667\\
270 87.66666666666667\\
280 88.5\\
290 78.16666666666666\\
300 81.33333333333333\\
310 80.5\\
320 85.83333333333331\\
330 94.83333333333333\\
340 90.0\\
350 90.16666666666667\\
360 79.83333333333334\\
370 97.33333333333333\\
380 97.33333333333334\\
390 94.0\\
400 87.16666666666667\\
410 98.33333333333331\\
420 90.66666666666666\\
430 93.16666666666666\\
440 95.0\\
450 98.83333333333333\\
460 98.00000000000001\\
470 94.66666666666667\\
480 96.83333333333334\\
490 98.00000000000001\\
500 98.5\\
};

\addplot [line width=0.4mm, GCNRNIFull,mark size=1pt]
table [row sep=\\]{%
10 51.33333333333333\\
20 52.16666666666667\\
30 51.83333333333333\\
40 51.33333333333333\\
50 53.49999999999999\\
60 51.33333333333334\\
70 53.333333333333336\\
80 53.16666666666667\\
90 52.83333333333334\\
100 52.5\\
110 53.166666666666664\\
120 54.0\\
130 57.33333333333333\\
140 59.33333333333333\\
150 53.333333333333336\\
160 53.0\\
170 58.33333333333333\\
180 56.49999999999999\\
190 52.16666666666667\\
200 53.99999999999999\\
210 52.0\\
220 56.99999999999999\\
230 56.33333333333332\\
240 53.166666666666664\\
250 54.0\\
260 54.50000000000001\\
270 51.16666666666667\\
280 53.166666666666664\\
290 54.50000000000001\\
300 53.0\\
310 52.5\\
320 53.99999999999999\\
330 54.666666666666664\\
340 53.33333333333334\\
350 56.333333333333336\\
360 55.00000000000001\\
370 52.33333333333333\\
380 54.833333333333336\\
390 54.333333333333336\\
400 55.833333333333336\\
410 55.49999999999999\\
420 54.50000000000001\\
430 54.16666666666667\\
440 56.50000000000001\\
450 55.16666666666667\\
460 57.49999999999999\\
470 53.833333333333336\\
480 57.49999999999999\\
490 54.83333333333334\\
500 57.666666666666664\\
};

\addplot [line width=0.4mm, GCNRNI50,mark size=1pt] %
table [row sep=\\]{%
10 51.16666666666667\\
20 51.5\\
30 53.666666666666664\\
40 50.66666666666666\\
50 52.166666666666664\\
60 51.16666666666667\\
70 54.49999999999999\\
80 53.83333333333334\\
90 53.49999999999999\\
100 53.833333333333336\\
110 54.33333333333332\\
120 55.666666666666664\\
130 57.333333333333336\\
140 58.166666666666664\\
150 57.00000000000001\\
160 59.5\\
170 58.50000000000001\\
180 64.0\\
190 63.16666666666666\\
200 65.83333333333333\\
210 73.16666666666667\\
220 81.33333333333333\\
230 86.16666666666667\\
240 89.16666666666666\\
250 93.66666666666667\\
260 93.66666666666666\\
270 94.33333333333334\\
280 94.16666666666667\\
290 95.83333333333333\\
300 95.33333333333334\\
310 97.16666666666667\\
320 96.83333333333331\\
330 98.00000000000001\\
340 98.66666666666667\\
350 98.16666666666667\\
360 98.83333333333333\\
370 97.83333333333331\\
380 98.66666666666667\\
390 99.33333333333334\\
400 99.16666666666669\\
410 99.16666666666669\\
420 98.83333333333333\\
430 99.16666666666666\\
440 99.16666666666666\\
450 99.49999999999999\\
460 99.33333333333334\\
470 98.83333333333333\\
480 99.0\\
490 97.33333333333334\\
500 99.33333333333334\\
};

\addplot [dotted, line width=0.6mm, GCNStandard, mark size=1pt]
table [row sep=\\]{%
0 50\\
10 50\\
20 50\\
30 50\\
40 50\\
50 50\\
60 50\\
70 50\\
80 50\\
90 50\\
100 50\\
110 50\\
120 50\\
130 50\\
140 50\\
150 50\\
160 50\\
170 50\\
180 50\\
190 50\\
200 50\\
210 50\\
220 50\\
230 50\\
240 50\\
250 50\\
260 50\\
270 50\\
280 50\\
290 50\\
300 50\\
310 50\\
320 50\\
330 50\\
340 50\\
350 50\\
360 50\\
370 50\\
380 50\\
390 50\\
400 50\\
410 50\\
420 50\\
430 50\\
440 50\\
450 50\\
460 50\\
470 50\\
480 50\\
490 50\\
500 50\\
};

\addplot [dotted, line width=0.6mm, GCNRNIFull, mark size=1pt]
table [row sep=\\]{%
10 50.66666666666666\\
20 51.0\\
30 50.66666666666666\\
40 52.5\\
50 52.66666666666667\\
60 55.666666666666664\\
70 51.16666666666666\\
80 57.16666666666666\\
90 56.49999999999999\\
100 56.99999999999999\\
110 56.666666666666664\\
120 61.5\\
130 67.83333333333333\\
140 69.16666666666667\\
150 72.16666666666669\\
160 76.16666666666666\\
170 79.5\\
180 78.16666666666666\\
190 79.83333333333333\\
200 81.66666666666667\\
210 82.99999999999999\\
220 85.0\\
230 84.00000000000001\\
240 85.83333333333334\\
250 82.33333333333333\\
260 87.16666666666667\\
270 88.33333333333334\\
280 88.16666666666666\\
290 88.83333333333333\\
300 88.83333333333333\\
310 88.66666666666664\\
320 87.16666666666667\\
330 89.0\\
340 86.33333333333333\\
350 91.83333333333333\\
360 89.66666666666666\\
370 92.66666666666666\\
380 90.16666666666666\\
390 90.66666666666666\\
400 90.0\\
410 90.16666666666666\\
420 92.83333333333333\\
430 91.16666666666666\\
440 92.83333333333333\\
450 90.99999999999999\\
460 90.99999999999999\\
470 93.33333333333333\\
480 92.83333333333331\\
490 92.16666666666666\\
500 90.66666666666666\\
};

\addplot [dotted, line width=0.6mm, GCNRNI50, mark size=1pt]
table [row sep=\\]{%
10 50.0\\
20 53.0\\
30 55.666666666666664\\
40 53.666666666666664\\
50 54.33333333333332\\
60 54.0\\
70 53.66666666666667\\
80 55.99999999999999\\
90 58.5\\
100 56.833333333333336\\
110 59.0\\
120 62.83333333333333\\
130 63.16666666666666\\
140 62.66666666666667\\
150 64.5\\
160 68.5\\
170 67.5\\
180 73.66666666666667\\
190 74.33333333333333\\
200 76.66666666666667\\
210 82.66666666666666\\
220 90.16666666666666\\
230 90.83333333333331\\
240 90.33333333333333\\
250 93.83333333333331\\
260 95.83333333333333\\
270 96.0\\
280 97.33333333333333\\
290 98.33333333333331\\
300 98.83333333333333\\
310 97.49999999999999\\
320 97.66666666666666\\
330 98.16666666666667\\
340 98.83333333333333\\
350 99.66666666666667\\
360 100.0\\
370 99.49999999999999\\
380 100.0\\
390 99.66666666666667\\
400 99.33333333333334\\
410 99.66666666666667\\
420 99.33333333333334\\
430 99.0\\
440 99.66666666666667\\
450 99.66666666666667\\
460 100.0\\
470 99.49999999999999\\
480 99.66666666666667\\
490 97.66666666666666\\
500 99.16666666666669\\
};

\end{axis}
\end{tikzpicture}
		\caption{Learning curves on \EXPTwo (/E) and \Corrupt (/C).}
		\label{app:fig:modelConvergenceSplit}
	\end{subfigure}%
	\break
	\begin{subfigure}{.48\textwidth}
		\centering
		\input{fig/fig_coreplusSparse}
		\caption{Learning curves on \SparseCEXP.}
		\label{app:fig:modelConvergencePlusSparse}
	\end{subfigure}%
	\begin{subfigure}{.48\textwidth}
		\centering
		\input{fig/fig_LRNEXPSparse}
		\caption{Learning curves on\EXPTwo (/E) and \Corrupt (/C).}
		\label{app:fig:modelConvergenceSplitSparse}
	\end{subfigure}%
	\caption{Model convergence results for Experiment 2 on \CEXP and \SparseCEXP.}
		\label{app:fig:res}
\end{figure*}

Analogously to Experiment 1, we generate a \SparseCEXP dataset similarly to \CEXP, but only generate 150 graph pairs. Then, we select 75 graph pairs and modify them, as described in \Cref{app:CExpConstruction}. We report the learning curves for all models on \SparseCEXP, as well as the original curves for \CEXP from the main body, in \Cref{app:fig:res}.

As in the previous subsection, similar behavior is observed on \SparseCEXP compared with \CEXP, only differing by slower convergence in the former case. However, we note that the ``struggle'' phase described in the main paper, which only occurs during the first 100 epochs over \CEXP, lasts for around 500 epochs on \SparseEXP. Intuitively, this ``struggle'' phenomenon is due to conflicting learning requirements, stemming from \Corrupt and \EXPTwo, which effectively require models to ``isolate''  deterministic dimensions  for \Corrupt, and other randomized dimensions for \EXPTwo. This in itself is already challenging on \CEXP, but is made even more difficult on \SparseCEXP due to its sparsity. Indeed, sparsity makes that further samples are needed in expectation to find a reasonable solution, leading to a lengthy ``struggle'' phase, in which both \Corrupt and \EXPTwo data points conflict with one another during optimization.

\subsection{Hyper-parameter details}
All GCN models with (partially or completely) deterministic initial node embeddings map a 2-dimensional one-hot encoding of node type (literal or disjunction) to a $k$-dimensional embedding space, where $k$ corresponds to the dimensionality of the deterministic embeddings. Furthermore, the final prediction for every graph is computed by aggregating all node embeddings using the $\max$ function, and then passing the result through a multi-layer perceptron (MLP) of 3 layers with dimensionality $x$, 32 and 2 respectively, where $x$ is the embedding dimensionality used in the given model. The activation function for the first two MLP layers is the ELU function \cite{ClevertUH15}, and the softmax function is used to make a final prediction at the final MLP layer.

All neural networks in this work are optimized using the Adam optimizer \cite{Kingma-ICLR2014}. All training is conducted with a fixed learning rate $\lambda$, for fairer comparison between all models. Initially, decaying learning rates were used, but these were discarded, as they yielded sub-optimal convergence for all \GCNRNI~models. Finally, all experiments were run on a V100 GPU. Detailed hyper-parameters, namely learning rate $\lambda$ and RNI distribution $p$, per model on every evaluation dataset are shown in \Cref{tab:hyperParams}.

\begin{table}[t!]
	\centering
	\caption{Hyper-parameter configurations for all experiments.} 
	\begin{tabular}{lcccc}
		\toprule 
		Dataset & \multicolumn{2}{c}{\EXP} & \multicolumn{2}{c}{\CEXP}\\
		  & $\lambda$ & $p$ & $\lambda$ & $p$\\
		 \cmidrule(r){2-3}
		 \cmidrule(r){4-5}
		 GCN & $1\times10^{-4}$ & N/A & $1\times10^{-4}$ & N/A \\
		 GCN-12.5\%RNI & $2\times10^{-4}$ & N & $2\times10^{-4}$ & N \\
		 GCN-50\%RNI & $2\times10^{-4}$ & N & $2\times10^{-4}$ & N \\
		 GCN-87.5\%RNI & $2\times10^{-4}$ & N & $5\times10^{-4}$ & N \\
		 \GCNRNI & $5\times10^{-4}$ & N & $5\times10^{-4}$ & N \\
		 \ThreeGNNFull & $5\times10^{-4}$ & N/A & $2\times10^{-4}$ & N/A \\
		\bottomrule
	\end{tabular}
	\label{tab:hyperParams}
\end{table}

\subsubsection{Results for \GCNRNI~with hyperbolic tangent activation}

\begin{table}[t!]
\centering
\caption{Performance of \GCNRNI on \EXP dataset with tanh.}
\begin{tabular}{lHc} 
\toprule
Model & Training Accuracy (\%) & Testing  Accuracy (\%)\\
\midrule
 \GCNRNI(U) & 93.0 $\pm$ 6.32 & 92.7 $\pm$ 5.61 \\ 
\textbf{\GCNRNI(N)} & \textbf{95.7 $\pm$ 2.64} & \textbf{96.0 $\pm$ 2.11} \\ 
\GCNRNI(XU) & 65.1 $\pm$ 20.6 & 64.6 $\pm$ 19.9 \\ 
\GCNRNI(XN) & 63.2 $\pm$ 21.2 & 63.0 $\pm$ 20.9 \\ 
\bottomrule
\end{tabular}
\label{tab:exp1Tanh}
\end{table}

In addition to experimenting with the RNI probability distribution, we also experimented with different activation functions for the GCN message passing iterations. Results are shown in \Cref{tab:exp1Tanh}. Performance with $tanh$ is significantly more variable across distributions than ELU, which shows that RNI is highly sensitive to  choices of hyper-parameters.

\end{document}